\documentclass[letterpaper]{article} 
\usepackage{aaai2026}
\usepackage{times}  
\usepackage{helvet}  
\usepackage{courier}  
\usepackage[hyphens]{url}  
\usepackage{graphicx} 
\usepackage{natbib}  
\usepackage{caption} 

\usepackage{algorithm}
\usepackage{algorithmic}

\usepackage{newfloat}
\usepackage{listings}
\DeclareCaptionStyle{ruled}{labelfont=normalfont,labelsep=colon,strut=off} 
\floatstyle{ruled}
\newfloat{listing}{tb}{lst}{}
\floatname{listing}{Listing}
\usepackage{amsmath,amssymb,amsthm}
\usepackage{booktabs}
\usepackage{multirow}
\usepackage{enumitem}
\usepackage{xcolor}
\usepackage[textsize=tiny]{todonotes}
\usepackage{mathtools}
\usepackage{bm}
\usepackage{cleveref}

\newtheorem{proposition}{Proposition}[section]

\newcommand{\pop}{\beta}
\newcommand{\npos}{\alpha}
\newcommand{\precision}{p}
\newcommand{\recall}{r}
\newcommand{\budget}{c}

\newcommand{\IEA}{\mathit{IE}^{\mathrm{A}}}
\newcommand{\IEB}{\mathit{IE}^{\mathrm{B}}}
\newcommand{\TPA}{T^{\mathrm{A}}}
\newcommand{\TPBOne}{T^{\mathrm{B,1}}}
\newcommand{\TPBTwo}{T^{\mathrm{B,2}}}
\newcommand{\TPB}{T^{\mathrm{B}}}
\newcommand{\rempop}{\beta_{\mathrm{rem}}}
\newcommand{\rempos}{\alpha_{\mathrm{rem}}}
\newcommand{\remprev}{\pi_{\mathrm{rem}}}
\newcommand{\rembud}{c_{\mathrm{rem}}}

\newcommand{\npert}{M}

\newcommand{\ntest}{n}
\newcommand{\nloop}{k}
\newcommand{\Dtst}{\mathcal{D}_{\mathrm{val}}}
\newcommand{\Dtstm}{\Dtst^{(m)}}

\newcommand{\p}{\mathcal{P}}
\newcommand{\A}{\mathcal{A}}
\newcommand{\fstar}{f^\star}

\newcommand{\Dval}{\mathcal{D}_{\mathrm{val}}}

\title{Intervention Efficiency and Perturbation Validation Framework: Capacity-Aware and Robust Clinical Model Selection under the Rashomon Effect}

\title{Intervention Efficiency and Perturbation Validation Framework: Capacity-Aware and Robust Clinical Model Selection under the Rashomon Effect}

\author{
    Yuwen Zhang\textsuperscript{\rm 1},
    Viet Tran\textsuperscript{\rm 2},
    Paul Weng\textsuperscript{\rm 1,\rm 3, \rm 4}\thanks{Corresponding author.}
}

\affiliations{
    \textsuperscript{\rm 1}Duke Kunshan University (DKU), Kunshan, China\\
    \textsuperscript{\rm 2}University College Dublin, Dublin, Ireland\\
    \textsuperscript{\rm 3}Digital Innovation Research Center, DKU, Kunshan, China\\
    \textsuperscript{\rm 4}Jiangsu Provincial University Key (Construction) Laboratory for Smart Diagnosis and Treatment of Lung Cancer\\
    yuwen.zhang@dukekunshan.edu,\ 
    quoc.tran@ucdconnect.ie,\ 
    paul.weng@dukekunshan.edu.cn
}

\usepackage{bibentry}

\begin{document}
\maketitle

\begin{abstract}
In clinical machine learning, 
the coexistence of multiple models with comparable performance---a manifestation of the Rashomon Effect---poses fundamental challenges for trustworthy deployment and evaluation. 
Small, imbalanced, and noisy datasets, 
coupled with high-dimensional and weakly identified clinical features, 
amplify this multiplicity and make conventional validation schemes unreliable. 
As a result, selecting among equally-performing models becomes uncertain, 
particularly when resource constraints and operational priorities are not considered by conventional metrics like F1 score.
To address these issues, 
we propose two complementary tools for robust model assessment and selection: 
the Intervention Efficiency (IE) and the Perturbation Validation Framework (PVF).
IE is a capacity-aware metric that quantifies how efficiently a model identifies actionable true positives when only limited interventions are feasible, 
thereby linking predictive performance with clinical utility. 
PVF introduces a structured approach to assess the stability of models under data perturbations, 
identifying models whose performance remains most invariant across noisy or shifted validation sets.
Empirical results on synthetic and real-world healthcare datasets show that using these tools facilitates the selection of models that generalize more robustly and align with capacity constraints,
offering a new direction of tackling the Rashomon Effect in clinical settings.
\end{abstract}


\begin{links}
    \link{Code}{https://github.com/YuwenZhang-Peter/PVF-IE}
\end{links}

\section{Introduction}\label{sec:introduction}

Small-scale clinical prediction modeling often faces challenges of limited sample size, 
class imbalance, and high dimensionality, with few established predictors. 
Data scarcity arises from high collection costs, ethical constraints, and follow-up difficulties \cite{coors2017ethical}. 
Class imbalance is common since adverse events affect only a small fraction of patients, 
leading to misleading results if conventional metrics like accuracy are used \cite{davis2021accuracy,he2009learning}. 
In such contexts, clinicians value precision and recall more than overall accuracy, 
especially when only a limited number of interventions can be performed \cite{martin2025use,li2025mortality,gurudevan2024lung}.

Because improving precision often reduces recall, 
an explicit trade-off is essential under constrained resources. 
The F1 score, as the harmonic mean of precision and recall, 
ignores true negatives and can favor suboptimal models when this trade-off is unclear \cite{he2009learning}. 
Threshold-independent metrics mitigate this issue but have drawbacks: AUC-ROC can overestimate performance on imbalanced data, while AUC-PR, though better for rare events, remains sensitive to prevalence and difficult to interpret clinically \cite{saito2015precision,richardson2024receiver}. Metrics such as balanced accuracy, Cohen’s kappa, and MCC have been proposed to address imbalance, 
with MCC noted for robustness \cite{chicco2020advantages}, 
yet all may yield ambiguous results requiring manual interpretation. 
Evaluating many models across several criteria is also computationally costly, 
motivating a measure that captures the precision–recall trade-off while reflecting capacity constraints.

High dimensionality further increases overfitting and obscures relevant predictors \cite{berisha2021digital,bommert2022benchmark},
motivating feature selection or dimensionality reduction \cite{jolliffe2016principal,jia2022feature}. 
Still, models like logistic regression, SVMs, k-nearest neighbors, 
and random forests can achieve similar accuracies while relying on distinct feature sets \cite{lukaszuk2024stability}. 
This phenomenon, called the Rashomon Effect \cite{breiman2001statistical,rudin2024amazing,xin2022exploring}, reveals that many “equally good” but mutually-exclusive models may coexist. 
This multiplicity usually complicated interpretation and clinical trust.

Model evaluation is also hindered by instability across data splits: 
small datasets amplify variance, 
and similar scores can occur by chance \cite{cawley2010over}. 
Traditional validation provides point estimates but neglects robustness, 
though model performance is known to fluctuate under realistic perturbations \cite{varoquaux2018cross,paschali2018generalizability,arleto2023robust}. 
Robustness to data shifts is now recognized as a core requirement for safe deployment in healthcare \cite{paschali2018generalizability,pfohl2019benchmarking,heinrich2022robustness}.

To address these challenges, we propose two complementary tools: 
the \emph{Intervention Efficiency (IE)} and the \emph{Perturbation Validation Framework (PVF)}. 
IE quantifies how effectively a model identifies minority cases when intervention capacity is limited,
making the precision–recall trade-off explicit and clinically meaningful. 
PVF stress-tests models across perturbed validation sets and selects the one with the most stable performance, 
emphasizing robustness and interpretability. 
While introduced alongside IE,
PVF is fully compatible with conventional metrics such as F1 score or accuracy. 
Experiments on synthetic and real medical datasets show that PVF guides model selection toward greater generalization under both IE and traditional metrics. 
Together, IE and PVF provide a resource-aware, robustness-focused evaluation framework that facilitates model selection when the Rashomon Effect appears in small, imbalanced healthcare datasets.

\section{Related Work}\label{sec:relatedwork}

This section discusses issues related to the Rashomon Effect. 
We (i) examine the pitfalls of relying on averaged metrics for small, imbalanced clinical datasets and motivate selecting a single deployable model (\ref{sec:single-deplayable-model}), 
and (ii) review perturbation-based validation for robust model selection while clarifying how our PVF differs from existing approaches (\ref{sec:perturbation-based-methods}).

\subsection{Average Metrics vs. Single Deployable Model}\label{sec:single-deplayable-model}

Small, imbalanced datasets create high variance, overfitting, and instability in model and feature selection \cite{cawley2010overfitting,zhang2023model,xin2022exploring}, 
making it difficult to select one reliable model among many, particularly when the Rashomon Effect appears and causes multiplicity in models. 
In low-sample settings, 
small data perturbations can cause drastic changes in model structure and predictions, 
especially for unstable learners such as decision trees or stepwise regression \cite{breiman1996bagging,breiman1996heuristics,morozova2015comparison,meinshausen2010stability,dietterich2000ensemble}. 
Under class imbalance, models may favor majority classes or overfit rare ones \cite{chawla2002smote,he2009learning,buda2018systematic,vandenGoorbergh2022harm}.

Cross-validation (CV) is widely used to estimate model performance but performs unreliably on small, high-dimensional data. 
Specifically, CV-based selection often yields optimistic estimates and unstable rankings across splits \cite{varma2006bias,cawley2010overfitting,arlot2010survey,varoquaux2018cross,zhang2021cross}. 
Nested CV can reduce bias but remains computationally expensive and unstable when samples are few \cite{wainer2018nested}. 
Moreover, conventional CV averages over models rather than assessing individual consistency.
Likewise, although ensemble and bagging methods reduce variance and improve generalization \cite{breiman1996bagging,breiman2001random}, 
they typically produce complex, less interpretable systems prone to overfitting \cite{dietterich2000ensemble,buhlmann2002analyzing,opitz1999popular}. 
In contrast, 
clinical use prioritizes transparent, single-model predictions for interpretability and deployment \cite{dong2020survey,valente2021improving,zadeh2023comparative}.

Our PVF can be applied on a certain train-validation split. 
Therefore, it can either be combined with cross-validation to improve the robustness of each validation, 
or as an alternative, 
help select a single deployable model. 
It can act as a promising approach to facilitate model selection when the Rashomon Effect appears.


\subsection{Perturbation-Based Validation in Small, Imbalanced Clinical Datasets}\label{sec:perturbation-based-methods}

Robustness evaluation often involves injecting random, non-adversarial noise. 
Classic studies added Gaussian noise to features to assess performance degradation \cite{egmont1997robustness,shafiee2019impact,parmar2015radiomic,yang2020ecg,pfohl2021impact,kulkarni2021gaussnoise}. 
Recent gene–expression experiments found that logistic regression retained the most stable feature sets, 
while Random Forests were least stable, 
despite comparable accuracy \cite{lukaszuk2024stability}. 
This shows that conventional metrics alone can overlook robustness to noise.

For interpretable models, robustness work is emerging. 
For example, loss-function redesign improves single-tree resilience to label noise \cite{sztukiewicz2024tree}, 
yet trees in small-sample settings can still memorize noisy outliers, 
and ensemble averaging may leave feature importance unstable \cite{soloff2024bagging}.

Perturbation-based validation explicitly tests stability against noise. 
Label-noise injection, as in Mutation Validation (MV) \cite{zhang2023model}, 
flips a subset of labels to measure performance drop.
Feature noise assesses covariate sensitivity \cite{egmont1997robustness,kulkarni2021gaussnoise}.
Resampling methods, such as bootstrap, repeated k-fold CV and bagging, 
perturb sample composition to produce performance distributions rather than single estimates \cite{breiman1996bagging,breiman2001random,soloff2024bagging}.

Our Perturbation Validation Framework (PVF) differs by applying feature-level noise only to the validation set, 
directly evaluating generalization under noisy conditions without retraining on corrupted data. 
Aggregating results from multiple perturbations identifies the model with the most consistent and stable predictions—a crucial aspect when addressing the Rashomon Effect.

We also considered label perturbation during validation. 
Prior work shows that robust models should withstand limited label noise in training \cite{zhang2023model}.
However, introducing label noise in validation can distort comparisons. 
Flipped labels disproportionately penalize stronger models (which previously predicted correctly) while leaving weaker ones largely unaffected, 
compressing or inverting true performance gaps. 
In highly imbalanced datasets, 
even a few flipped labels can drastically alter metrics like precision and recall. 
Therefore, PVF excludes label perturbation in validation, despite its popularity for robustness testing during training.

\section{Methodologies}
\label{sec:methodologies}


This section has two parts: 
\Cref{sec:ie} proposes \emph{Intervention Efficiency} (IE), 
comparing model-guided and random interventions under limited capacity; 
\Cref{sec:pvf} introduces the \emph{Perturbation Validation Framework} (PVF) for robust model assessment and model selection.

\subsection{Intervention Efficiency (IE)}
\label{sec:ie}

In many applied domains such as clinical follow-up, fraud investigation, and social services, 
only a fixed budget \(\budget\) of interventions can be performed. 
In practice, 
this means that only a subset of the at-risk population can be targeted due to limited capacity.

A benchmark policy, 
referred to as \emph{uniform intervention}, 
randomly selects \(\budget\) individuals from the population without any predictive guidance. 
In contrast,
a \emph{model-guided intervention} uses a binary classifier to prioritize individuals for action. 
Two situations may occur:
\begin{enumerate}
    \item The model predicts more positives than \(\budget\), and only the top \(\budget\) flagged cases can be intervened upon.
    \item The model predicts fewer positives than \(\budget\), and after acting on all predicted positives, the remaining capacity is allocated through uniform intervention on the rest of the population.
\end{enumerate}

Because uniform intervention represents a fixed baseline under identical problem setting and capacity constraint, 
it provides a natural reference for evaluating model-guided strategies. 

We define the \emph{Intervention Efficiency (IE)} as the ratio between the expected number of true positives captured by a model-guided intervention and that captured by uniform intervention, 
given an intervention capacity \(\gamma = \budget / \pop\), 
where \(\pop\) denotes the population size. 
IE thus quantifies how effectively a predictive model utilizes limited resources to identify positive cases: 
higher IE indicates a greater number of true positives captured within the same budget.

In the following proposition, 
we show that the IE of a model \(f\) depends only on its precision \(\precision\), 
recall \(\recall\), 
the prevalence rate \(\pi\) (i.e., the percentage of positive cases in the population), 
and the intervention capacity \(\gamma\).

\begin{proposition}[Closed form of \(\mathrm{IE}\)]
\label{prop:IE-closed-form}
Assume that there is a model \(f\) that can predict binary labels and it is used for model-guided intervention. The problem has an overall prevalence rate \(\pi\) and only a fraction \(\gamma\) of the whole population can be intervened due to limited capacity. The Intervention Efficiency of \(f\) on this problem under intervention capacity \(\gamma\), denoted as \(\mathrm{IE}_\gamma(f)\), satisfies
\[
  \mathrm{IE}_\gamma(f)
  =
  \frac{s \precision + (\gamma - s)\,\dfrac{\pi - s \precision}{1 - s}}{\gamma \pi},
\]
where \(\precision\) is the precision of \(f\) and \(s = \min\!\bigl(\gamma, \tfrac{\pi \recall}{\precision}\bigr)\), with \(\recall\) being the recall of \(f\).
\end{proposition}

We refer readers to \Cref{proof of prop:IE-closed-form} for the detailed proof of \Cref{prop:IE-closed-form}. 
In practice, we can approximate the true $\pi$ by the prevalence rate $\hat{\pi}$ in the specific set of samples to which the model is applied.
For instance, 
when estimating IE on a validation set, 
even with the whole dataset that can provide a "more accurate estimation" for $\pi$, 
we should use the $\hat{\pi}$ from the specific validation set rather than the prevalence rate from the whole dataset. 
The reason is that IE is tied to precision and recall, 
two values depending not only on the model but also largely on the set of samples that yields them, especially in small sample cases.
In contrast, in large sample cases, 
$\hat{\pi}$ has been a good estimate of true $\pi$ given enough samples.
In sum, it is both reasonable and essential to estimate $\pi$, $\precision$ and $\recall$ on the same set to ensure they are aligned, 
since they are jointly dependent on each other.
Practically, when computing IE on a specific set of samples $\mathcal{D}$, 
it not only depends on a model $f$ and a $\gamma$ of interest, 
but also on $\mathcal{D}$, where $\hat{\pi}$, $\hat{\precision}$ and $\hat{\recall}$ (estimators of $\pi$, $\precision$ and $\recall$) are derived.
Consequently, we can rewrite the formula for IE in real practice as:
\[
\mathrm{IE}_\gamma(f, \mathcal{D})
=
\frac{\hat{s} \hat{\precision} + (\gamma - \hat{s})\,\dfrac{\hat{\pi} - \hat{s} \hat{\precision}}{1 - \hat{s}}}{\gamma \hat{\pi}},
\]
where $\mathcal{D}$ is a set of samples the model $f$ is applied on for evaluation,
$\hat{\pi}$ is the prevalence rate in $\mathcal{D}$, 
$\hat{\precision}$ is the precision of $f$ when evaluated on $\mathcal{D}$,
and $\hat{s} = \min\!\bigl(\gamma, \tfrac{\hat{\pi} \hat{\recall}}{\hat{\precision}}\bigr)$, 
with $\hat{\recall}$ being the recall of $f$ when evaluated on $\mathcal{D}$.

As for $\gamma$, 
to deploy a model trained on distribution 1 in distribution 2, 
the IE used during model selection on distribution 1 should be equipped with the $\gamma$ corresponding to distribution 2. 
This aligns selection with the intervention capacity of the target setting, improving generalization to the distribution 2 user case.

\subsection{Perturbation Validation Framework (PVF)}
\label{sec:pvf}

During model selection, the goal is not only to identify a model that performs well on the validation set, but also one that remains robust under practical noise that may be present in the data. Perturbed Validation Framework (PVF) provides a protocol for robustness assessment and model selection.
Starting from a fixed validation set, it generates multiple perturbed versions via adding noise on each sample, evaluates each candidate model on every perturbed set using a chosen metric, and aggregates the resulting scores into a single robustness-adjusted criterion.

\subsubsection{Setup}
Let the original validation set be
\(
  \Dval=\{(\bm x_i,y_i)\}_{i=1}^{n}
\),
where each $(\bm x_i,y_i), i = 1, 2, ..., n$ is a multi-dimensional sample with features $\bm x_i$ and a binary label $y_i$,
while $n$ is the total number of samples in $\Dval$.
Let \(\mathcal F\) be a finite set of trained candidate models.
The goal is to select the most robust model $f\in\mathcal{F}$, denoted as $f^\star$, based on the model performance on $\Dval$.

\subsubsection{Perturbation mechanism} \label{subsubsec:perturbation-mechanism}
Each feature of each sample is perturbed independently based on its type, with labels unchanged. Below, we give a couple of examples of perturbation.
\begin{enumerate}
  \item For numeric feature \(x\) in \(\bm x\),  
        add small scalar noise drawn from a user-specified distribution.  
        For example, in our experiments, we draw \(\varepsilon\sim\mathcal N(0,\sigma^2)\) and set \(x' = x + \varepsilon\).

  \item For categorical (nominal) feature \(x\) in \(\bm x\), 
        with a user-controlled flip probability, change \(x\) to another category sampled from a specified mass function over the remaining categories.  
        For example, in our experiments, with probability \(1-\xi\), keep \(x'=x\). Otherwise, choose \(x'\neq x\) uniformly from the remaining categories, i.e., 
        \[
        \Pr(x'=a\mid x)
        =
        \xi \cdot \frac{1}{(|\mathcal C|-1)}
        \]
        for \(a\neq x\), where \(|\mathcal C|\) is the number of categories feature \(x\) can have.

  \item For ordinal feature \(x\) in \(\bm x\), 
        move along a predefined order with probability mass concentrated on nearby categories, as controlled by a decay/temperature parameter.  
        For example, in our experiments, with probability \(1-\xi\), keep \(x'=x\). Otherwise, for a pre-determined distance \(d(\cdot,\cdot)\) on the order, sample \(x'\neq x\) with probability
        \[
          \Pr(x'\mid x)
          =
          \xi \cdot \frac{\exp\{-\lambda\,d(x,x')\}}
               {\sum_{a\neq x}\exp\{-\lambda\,d(x,a)\}}, 
        \]
        where \(a\) is any order \(x\) can have and \(\lambda\) is the decay parameter.
\end{enumerate}

\subsubsection{Constructing perturbed validation sets}
Suppose \(M\) perturbed validation sets are desired. Then, for each \(m=1,\dots,M\), form the \(m\)-th perturbed validation set \(\Dval^{(m)}\) by generating exactly \(k\) independent replicas of every \((\bm x_i,y_i)\) using the mechanism above. Thereby, \(|\Dval^{(m)}|=kn\). The reason for having an integer \(k\)  here is to make sure every sample is perturbed by equal times, making the overall distribution and class imbalance preserved. Also, \(k\) should be sufficiently large in order to let various perturbed versions of the original sample emerge in the perturbed validation sets, given different perturbation patterns for different features and a resulting complex composition of them.

\subsubsection{Evaluation, aggregation, and selection}
Let \(\mathrm{Eval}\) be any scalar metric that maps a model and a dataset to a performance score. 
For each candidate model \(f\in\mathcal F\) and \(m=1,\dots,M\), the performance of \(f\) on the \(m\)-th perturbed validation set \(\Dval^{(m)}\) can be expressed as
\(
P_f^{(m)}=\mathrm{Eval}\bigl(f,\Dval^{(m)}\bigr)
\).
Aggregate \(\bm P_f=(P_f^{(1)},\dots,P_f^{(M)})\) via \(\mathcal A\) to obtain \(A_f=\mathcal A(\bm P_f)\), where  \(A_f\) is the aggregated performance score of \(f\) on the \(M\) perturbed validation sets and \(\mathcal A\) is a user-specified aggregation function (e.g., the 25-th percentile of the \(M\) scores). The overall workflow for evaluating a single model is summarized in \Cref{fig:pvf-workflow}.

\begin{figure*}[ht]
  \centering
  \includegraphics[width=\textwidth]{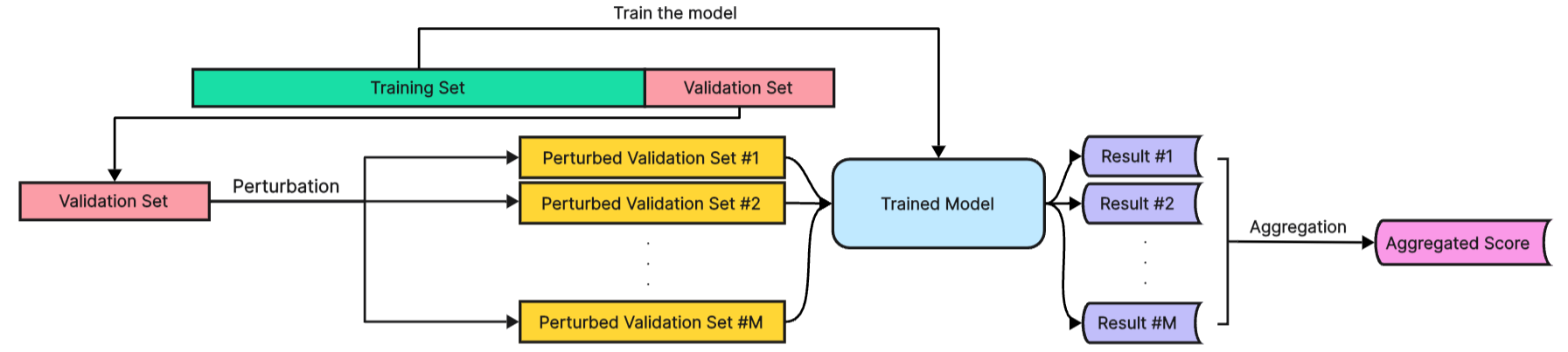}
  \caption{Overview of the Perturbation Validation Framework (PVF). After training a model on the training split, the original validation set is repeatedly perturbed (by adding noise to numeric features, changing categories for nominal features, or changing orders for ordinal features) to yield \(M\) validation sets. The trained model is evaluated on each perturbed validation set, producing \(M\) scores, which are then aggregated into a single robustness metric score.}
  \label{fig:pvf-workflow}
\end{figure*}

\begin{figure*}[ht]
  \centering
  \includegraphics[width=\textwidth]{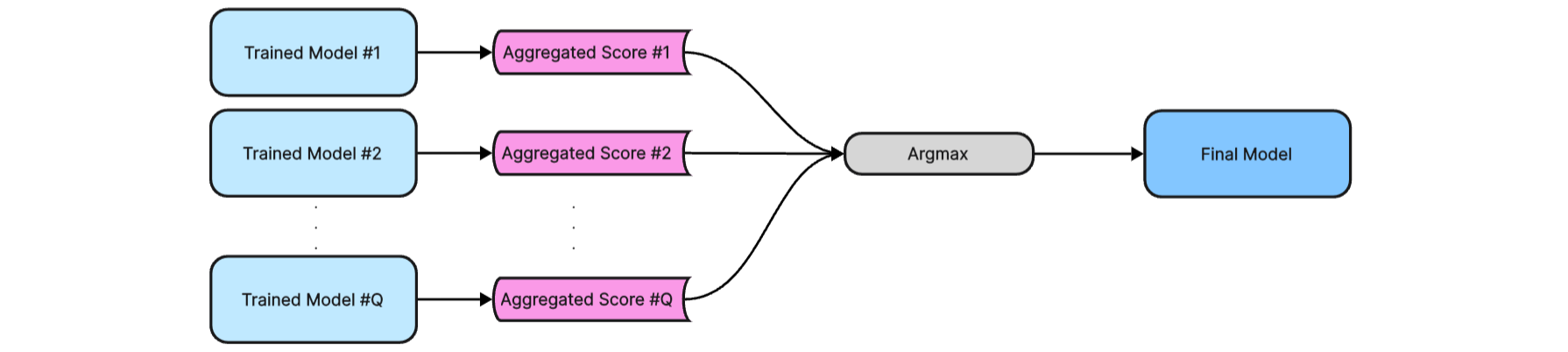}
  \caption{Model selection stage. Each candidate \(f\) yields a final aggregated score via PVF. The model \(f^\star\) with the highest aggregated score is chosen.}
  \label{fig:pvf-selection}
\end{figure*}

Finally, select
\(
  f^\star\in\arg\max_{f\in\mathcal F} A_f
\),
as shown in \Cref{fig:pvf-selection}. The complete implementation pseudocode of combining PVF with any evaluation metric \(\mathrm{Eval}\), denoted as PVF-\(\mathrm{Eval}\), appears in \Cref{alg:pvf-eval}.

\subsubsection{Hyperparameters and user-specified functions}\label{pvf-table}

For clarity, we provide the following list of components in PVF that can be changed according to specific user cases:

\begin{enumerate}[label=(\roman*), leftmargin=*, itemsep=2pt]
  \item \(d\): number of (or, which) features to be perturbed.
  \item \(k\): number of replicas per original sample (so each perturbed set has size \(kn\)).
  \item \(M\): number of perturbed validation sets.
  \item A collection of perturbation controls (e.g., noise standard deviation, flipping probabilities).
  \item \(\mathcal A:\mathbb R^M\to\mathbb R\): aggregation function (e.g., a lower quantile).
\end{enumerate}

Since no retraining is required, 
the per‐model computational complexity is \(\mathcal O(Mknd)\), 
and across all \(Q\) models it is \(\mathcal O(QMknd)\).
We refer readers to \Cref{theoretical guarantees of pvf} for theoretical guarantees of PVF.

\begin{algorithm}
  \caption{PVF-\(\mathrm{Eval}\) Model Selection}
  \label{alg:pvf-eval}
  \begin{algorithmic}[1]
    \REQUIRE Original validation set $\Dval=\{(\bm{x}_i,y_i)\}_{i=1}^{n}$;
             candidate model set $\mathcal{F}$;
             evaluation metric \textsc{Eval};
             number of perturbed sets $M$;
             replicas per sample $k$;
             a collection of perturbation controls;
             aggregation function $\mathcal{A}$
    \ENSURE Selected most robust model $f^\star$ on $\Dval$
    \FORALL{$f \in \mathcal{F}$}
      \FOR{$m=1$ \TO $M$}
        \STATE Build $\Dval^{(m)}$ by independently perturbing each feature
               of each $(\bm{x}_i,y_i)$ as specified in Section~\ref{sec:pvf},
               creating $k$ replicas per original sample.
        \STATE $P_f^{(m)} \gets \textsc{Eval}\bigl(f;\Dval^{(m)}\bigr)$
      \ENDFOR
      \STATE $A_f \gets \mathcal{A}\bigl(P_f^{(1)},\ldots,P_f^{(M)}\bigr)$
    \ENDFOR
    \STATE $f^\star \gets \arg\max_{f\in\mathcal{F}} A_f$
    \RETURN $f^\star$
  \end{algorithmic}
\end{algorithm}

\section{Experiments}
\label{sec:experiments}

This section outlines the objectives, general setups, and overall findings of our experiments on both synthetic and real datasets. Detailed descriptions, full quantitative results, and sensitivity analyses are provided in the Appendix. Specifically, \Cref{sec:synthetic-setups-results} present the detailed setups and results for the synthetic dataset experiments; \Cref{sec:real-setup} presents the detailed setups for real datasets; \Cref{sec:sensitivity-analysis} reports a comprehensive hyperparameter sensitivity analysis.

\subsection{Research Questions}
To evaluate the effectiveness and robustness of our proposed method, we designed a set of experiments guided by the following four research questions:

\begin{itemize}
    \item[\textbf{Q1.}] Can the proposed PVF be more effective than the traditional method (model selection based on a single held-out validation set) under traditional metrics (here, F1 Score and accuracy) on synthetic data (\Cref{sec:synthetic-results-brief} and \Cref{sec:synthetic-results})?
    \item[\textbf{Q2.}] Is PVF more effective than the traditional method under IE across different intervention fractions ($\gamma$) on synthetic data (\Cref{sec:synthetic-results-brief} and \Cref{sec:synthetic-results})?
    \item[\textbf{Q3.}] On real clinical datasets, does PVF select models that perform better than those selected by the traditional approach under the same metrics used in the synthetic dataset experiments (\Cref{sec:real-results})?
    \item[\textbf{Q4.}] How sensitive is PVF to its hyperparameters, with a focus on the perturbation controls (\Cref{sec:sensitivity-analysis})?
\end{itemize}

These questions allow us to compare PVF against the traditional validation method, examine its performance under both synthetic and real clinical data, and assess its robustness to experimental settings.

\subsection{General Experimental Setup}

In this section, 
we introduce the general experimental setups for both the synthetic dataset experiments and the real dataset experiments, 
along with the rationales behind the designs.

We design our experiments around the task of model selection---choosing a single model from a pool of candidates based on its validation performance. This setting commonly arises in the presence of the Rashomon Effect. The aim is to evaluate whether PVF can yield more reliable choices than a conventional single-split validation. The reason for not including more sophisticated frameworks like cross-validation is that PVF can be applied on any single-split train-validation settings, thus can be combined with these frameworks, as stated in \Cref{sec:single-deplayable-model}.

Evaluation is organized by fixing one metric at a time and comparing PVF against the traditional method under that metric. Intervention Efficiency (IE) is reported at different intervention capacity $\gamma$ to represent different clinical resource levels. We also report F1 Score (or supplemented with accuracy) as a complementary analysis, given that these two traditional metrics are widely adopted and emphasized in many clinical applications. Note that the focus is on which selection strategy chooses better models given a chosen metric, instead of whether IE or any traditional metric (F1 Score, accuracy) is better for model selection, as the choice of performance metric is case-specific. 

After selection, the chosen model is either assessed based on ground truth information or when it is not available, as an approximation of it, on a large external test set using the same metric as in validation. To improve reliability and reduce variance from any single split, each configuration is repeated extensively. Model classes and other dataset-specific design choices are kept simple and interpretable to mimic clinical use  \cite{stiglic2020interpretability}, 
and are described in subsequent subsections.

\subsection{Overview of Experimental Results}
\label{sec:overview-results}

This section presents an overview of our key findings, highlighting the probability of PVF outperforming the traditional validation method across all synthetic experimental configurations, as well as its benefits on real datasets after parameter tuning. Comprehensive descriptions of the experimental settings, configurations, detailed results, and sensitivity analyses are provided in the Appendix.

\subsubsection{Synthetic Dataset Experiments}
\label{sec:synthetic-results-brief}

As shown in \Cref{fig:synthetic-overview}, PVF consistently outperforms the traditional validation method across most configurations (the configurations are elaborated in \Cref{sec:synthetic-setup}). 
When $\gamma = 0.1$ and $\gamma = 0.3$, PVF achieves superior selections in approximately 90\% of all configurations, 
indicating PVF's advantage under limited intervention capacities. 
Although the advantage narrows as $\gamma$ increases ($\gamma \ge 0.7$), PVF still maintains a clear lead.

\begin{figure}[t]
    \centering
    \includegraphics[width=\linewidth]{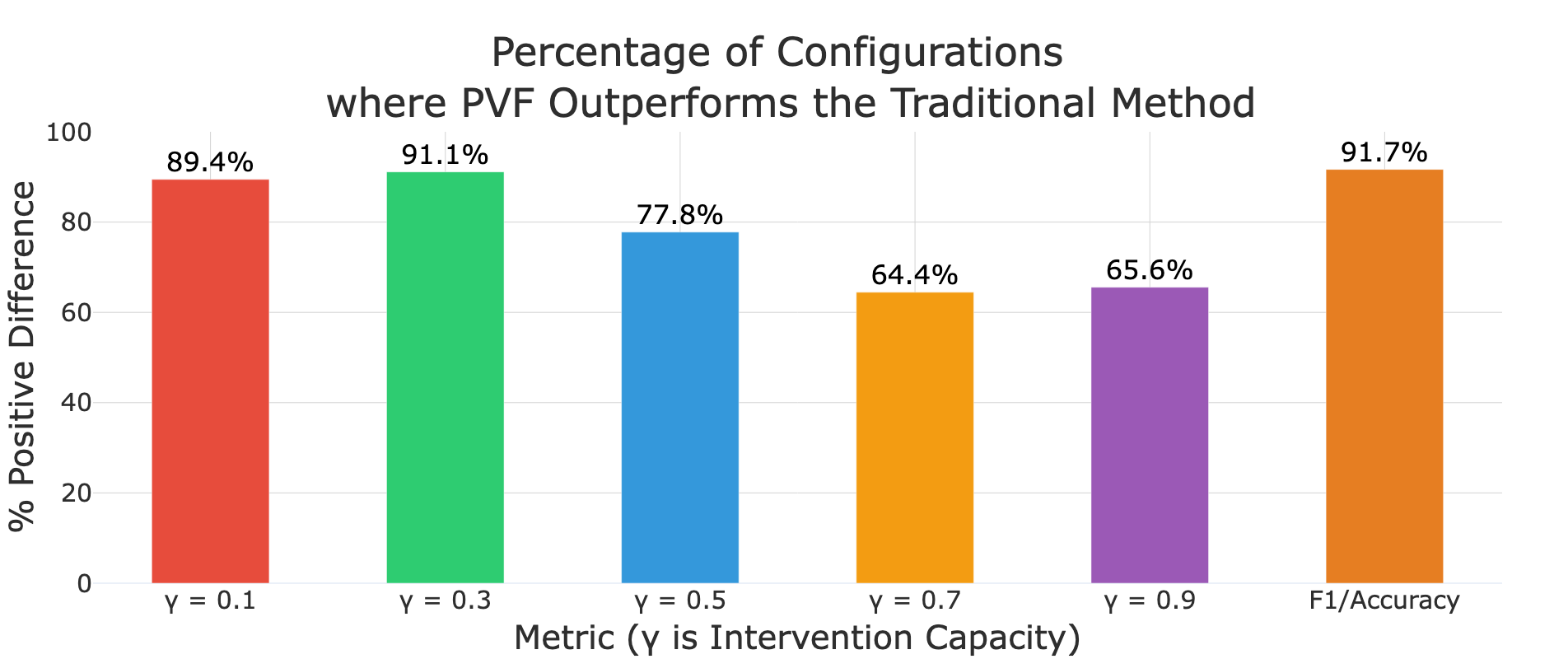}
    \caption{Percentage of configurations in synthetic experiments where PVF outperforms the traditional method across all metrics. "\% Positive Difference" shows the percentage of configurations in which PVF more often selects the best model than the traditional method does across all repetitions.}
    \label{fig:synthetic-overview}
\end{figure}

\subsubsection{Real-World Dataset Experiments}

\label{sec:real-results}

In this subsection, 
similar to the synthetic dataset experiment,
we also present the experimental results in the form of a comparison between the performance of PVF under $\sigma$ (defined as in \Cref{subsubsec:perturbation-mechanism}) tuning and that of the traditional method,
respectively for the cervical cancer \cite{wolberg1993breast} and breast cancer \cite{fernandes2017cervical} datasets.
We refer the readers to \Cref{sec:real-setup} for the detailed setups of this part of experiments,
including exactly how we calculate the performance gaps between PVF and the traditional method with repetitive experiments.

\paragraph{Results on cervical cancer data}
For the cervical cancer dataset, 
PVF achieves its largest performance gap with the traditional method with a small perturbation for IE ($\sigma = 0.01$) and a near-zero perturbation for F1 Score ($\sigma = 10^{-6}$), 
as shown in \Cref{fig:cervical-best}. 
At the most stringent intervention capacity ($\gamma = 0.1$), 
PVF selects an externally better model in 60.0\% of times, 
compared to 26.7\% for the traditional method and 13.3\% ties. 
As $\gamma$ increases to 0.3–0.9, 
PVF continues to win more frequently, 
though the margin narrows: 
43.3–46.7\% times PVF wins versus 33.3–36.7\% the traditional method does,
with around 20\% ties across these settings. 
For F1, PVF still leads with 50.0\% wins, 
compared to 33.3\% the traditional method does and 16.7\% ties. 
These results indicate that once $\sigma$ is tuned, 
PVF consistently yields more reliable model selection than the traditional method, 
with its strongest advantage at $\gamma = 0.1$, 
where false positives are extremely costly.

\paragraph{Results on breast cancer data}
For the breast cancer dataset, 
the perturbation scale that works best is much larger ($\sigma = 0.2$–$0.3$), 
as shown in \Cref{fig:breast-best}.
At $\gamma = 0.1$,
PVF wins in 52.0\% of folds while the traditional method accounts for only 20.0\%, 
with 28.0\% ties. 
At $\gamma = 0.3$, 
the pattern is similar: 
48.0\% PVF, 20.0\% traditional, 
and 32.0\% ties. 
At higher intervention fractions ($\gamma = 0.5, 0.7, 0.9$) with $\sigma = 0.2$, 
results are more balanced, 
with PVF and the traditional method each winning 40.0\% of folds and 20.0\% ties. 
For F1 Score, 
the best $\sigma$ is 0.2, 
and PVF wins in 52.0\% of folds compared to 16.0\% naive and 32.0\% ties.

\paragraph{Real dataset result summary}
The perturbation level required for PVF to be most effective differs across datasets and evaluation metrics, 
but once tuned, PVF consistently outperforms or matches the traditional approach, 
with the largest benefits again observed at lower $\gamma$'s.

\subsubsection{Brief Conclusion of Experiments}
Across synthetic and real-data studies, 
PVF generally more often selects better models than the traditional single-split procedure. 
On synthetic datasets, 
PVF consistently improves selection under the traditional metrics (F1/accuracy) and IE, 
with the largest gains at smaller intervention capacities ($\gamma=0.1,0.3$) and F1 score/accuracy. 
On real datasets, 
PVF retains a clear advantage once the perturbation level is tuned, 
again with the strongest improvements at smaller capacities and F1 score.
\Cref{sec:sensitivity-analysis} further shows that $\sigma$ is a key hyperparameter,
and there is no universal best scale for different real datasets. 

\begin{figure}[tbp]
    \centering
    \includegraphics[width=0.47\textwidth]{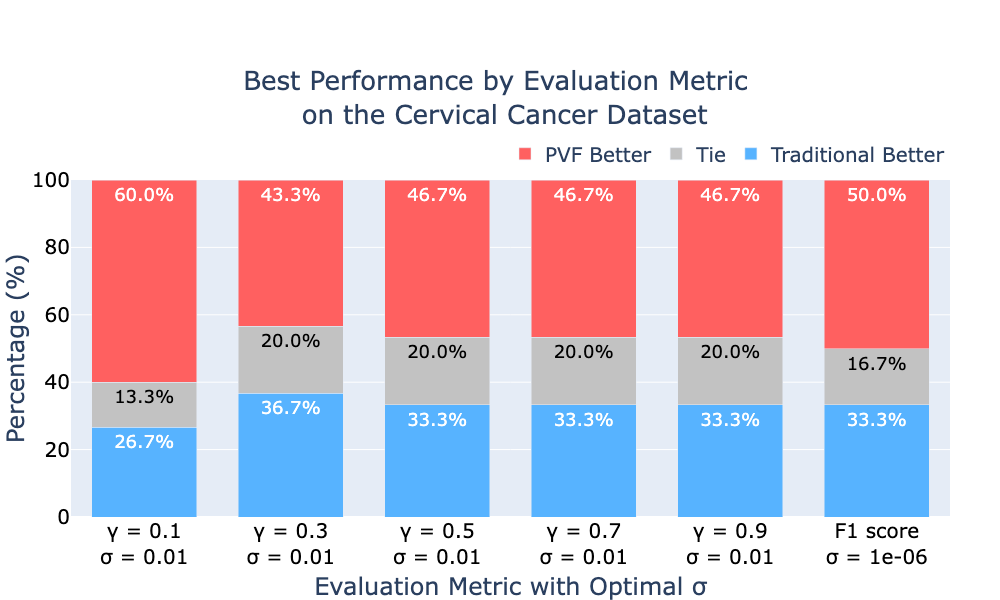}
    \caption{Largest performance gaps of PVF on the cervical cancer dataset under different evaluation metrics, when $\sigma$ is tuned. PVF consistently selects better external models than the traditional method across all $\gamma$ values, with its strongest advantage at $\gamma = 0.1$.}
    \label{fig:cervical-best}
\end{figure}

\begin{figure}[tbp]
    \centering
    \includegraphics[width=0.47\textwidth]{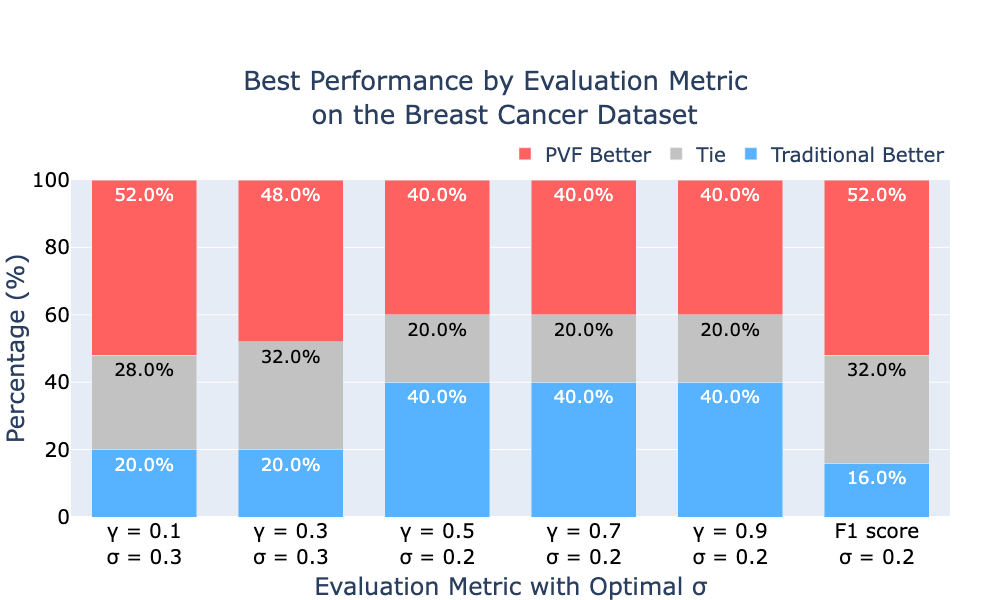}
    \caption{Largest performance gaps of PVF on the breast cancer dataset under different evaluation metrics, when $\sigma$ is tuned. PVF tends to outperform the traditional method across lower $\gamma$ values, specifically when $\gamma \in \{0.1, 0.3\}$, while tying with the traditional method at higher $\gamma$ values.}
    \label{fig:breast-best}
\end{figure}

\section{Conclusion}~\label{sec:conclusion}

This work aim to address a persistent challenge in clinical machine learning: selecting a single, interpretable model when data are limited, class distributions are imbalanced, and measurements are noisy, potentially aggravating the Rashomon Effect. We propose \emph{Intervention Efficiency} (IE), a capacity-aware evaluation metric that quantifies the expected benefit of a model when only a fixed fraction of individuals can receive an intervention, and \emph{Perturbation Validation Framework} (PVF) that prioritizes candidate models whose performance remains strong under perturbations of the validation set. Together, IE and PVF facilitate model selection under operational constraints and robustness requirements that are typical of clinical deployment.

Empirical studies spanning an exhaustive synthetic enumeration and two clinical datasets indicate that the PVF more often identifies models that generalize better out of sample than a conventional single-split selection does. The gains are most pronounced in the regimes with a low intervention capacity and when F1 score is the metric of interest.

As for limitations and future work, the selection and calibration of perturbation noise for PVF are consequential and may require domain-expert input to ensure that perturbations faithfully reflect clinical measurement variability and data quality. Otherwise, PVF may fail to outperform the traditional validation approach. Therefore, developing principled elicitation methods (e.g., protocol- or instrument-informed priors) and practical defaults for common data types would be crucial to improve applicability. On the theory side, there is substantial scope for formal analysis of PVF, including consistency guarantees and sensitivity to different performance metrics, especially the newly proposed IE. Such results would clarify applicable scenarios defined by data scale, noise, and resource limits. For IE, extensions to multiclass outcomes, integration of explicit cost, and consideration of fairness constraints remain important directions. Overall, coupling PVF with IE offers a practical, conceptual pathway and a future direction for navigating the Rashomon Effect in clinical machine learning, helping translate model multiplicity from a source of uncertainty into a framework for robust and resource-conscious model selection.

\section{Acknowledgments}

This work was supported by the Data+X (2024) project at DKU. We would also like to thank Prof. Dongmian Zou for valuable discussions on the theoretical aspects of the Intervention Efficiency and the Perturbation Validation Framework, and Prof. Chenkai Wu for his insightful suggestions on the Intervention Efficiency that helped us connect the proposed metric to practical interpretations.

\bibliography{aaai2026}

\appendix

\section{Proof of \Cref{prop:IE-closed-form}}
\label{proof of prop:IE-closed-form}
\begin{proof}
We present the derivation in five steps including the analysis of the model's capability, the two operating regimes, the expression unification for the two regimes, and the transition from counting-based expression to the final ratio-based one. All notations and terms follow the same definitions in \Cref{prop:IE-closed-form}. For clarity, we additionally denote \(\npos\) as the total number of positives in the population, \(\pop\) as the population size, and \(\budget\) as the number of individuals that can be intervened. Specifically,
\[
\pi = \frac{\npos}{\pop}, 
\quad 
\gamma = \frac{\budget}{\pop}.
\]

\begin{enumerate}[label=\textbf{(\roman*)}, leftmargin=*, itemsep=2pt]
  \item \textbf{Model capability.}
        We first calculate how many individuals the model flags positive to attain the number of true positives. In the following analysis, \(TP\), \(FP\), \(TN\), \(FN\) denote true positives, false positives, true negatives and false negatives in the confusion matrix.
        Since \(\recall = TP/\npos\), the model can capture at most \(TP=\npos\,\recall\) positives.
        By the definition of precision,
        \[
          \precision = \frac{TP}{TP+FP}
          \quad\Longrightarrow\quad
          TP+FP=\frac{TP}{\precision}=\frac{\npos\,\recall}{\precision}.
        \]
        Therefore, the model labels \(\tfrac{\npos\,\recall}{\precision}\) individuals as positive and captures \(\npos\,\recall\) true positives in the whole population. Note that we may not be able to intervene on all these \(\tfrac{\npos\,\recall}{\precision}\) cases due to the fixed budget \(\budget\), and this creates the two following regimes.

  \item \textbf{Regime A (scarce resources).}
        If \(\tfrac{\npos\,\recall}{\precision}>\budget\), only \(\budget\) individuals can be intervened, which is insufficient for the model to realize its theoretical best in capturing all \(\npos\,\recall\) positives.
        Each model-guided intervention yields a true positive with probability \(\precision\), so the expected number of positives captured by the model in this regime, denoted as \(\TPA\), satisfies
        \[
          \TPA=\budget\,\precision.
        \]
        By comparison, random selection yields \(T_{\mathrm{random}} = \budget\,(\npos/\pop)\) true positives in expectation. Hence, by definition, the IE value in this regime satisfies
        \[
          \IEA
          =
          \frac{\TPA}{T_{\mathrm{random}}}
          =
          \frac{\budget\,\precision}{\;\budget\,(\npos/\pop)\;}
          =
          \frac{\pop\,\precision}{\npos}.
        \]

  \item \textbf{Regime B (ample resources).}
        If \(\tfrac{\npos\,\recall}{\precision}\le\budget\), the budget suffices to treat all model-flagged positives. 
        We intervene through a two-stage procedure.
        In stage 1,
        intervening on those \(\tfrac{\npos\,\recall}{\precision}\) cases predicted positive by the model captures \(\TPBOne=\npos\,\recall\) positives.
        The remaining budget is \(\rembud=\budget-\tfrac{\npos\,\recall}{\precision}\).
        The remaining pool has
        \[
          \rempop=\pop-\tfrac{\npos\,\recall}{\precision},\quad
          \rempos=\npos-\npos\,\recall,\quad
          \remprev=\frac{\rempos}{\rempop},
        \]
        where \(\rempop\) denotes the size of the residual population after removing the \(\tfrac{\npos\recall}{\precision}\) intervened, model-flagged individuals; \(\rempos\) is the number of true positives remaining in that residual population; and \(\remprev\) is the corresponding prevalence rate within this residual pool.
        We assume that the predictive model is uninformative on this residual pool. Therefore, in stage 2, the remaining \(\rembud\) interventions are allocated uniformly at random, yielding
        \[
          \TPBTwo=\rembud\,\remprev
        \]
        additional expected true positives.
        Consequently, the total number of positives captured by the model in regime B, denoted as \(\TPB\), satisfies
        \[
          \TPB
          =
          \TPBOne+\TPBTwo
          =
          \npos\,\recall
          +
          \Bigl(\budget - \tfrac{\npos\,\recall}{\precision}\Bigr)\,
          \dfrac{\npos - \npos\,\recall}{\;\pop - \tfrac{\npos\,\recall}{\precision}\;},
        \]
        and the IE value in this regime satisfies
        \[
          \IEB
          =
          \dfrac{
            \npos\,\recall
            +
            \Bigl(\budget - \tfrac{\npos\,\recall}{\precision}\Bigr)\,
            \dfrac{\npos - \npos\,\recall}{\;\pop - \tfrac{\npos\,\recall}{\precision}\;}
          }{
            \budget\,\dfrac{\npos}{\pop}
          }.
        \]

  \item \textbf{Unified expression.}
        To avoid case-by-case handling, set
        \(
          c'=\min\!\bigl(\budget,\tfrac{\npos\,\recall}{\precision}\bigr).
        \)
        Then, the expected positives captured respectively by the model,
        denoted as \(T_{\mathrm{model}}\), 
        and by random selection, are
        \[
          T_{\mathrm{model}}
            = c'\,\precision
              + (\budget - c')\,
                \frac{\npos - c'\,\precision}{\pop - c'}, 
          \qquad
          T_{\mathrm{random}}
            = \budget\,\frac{\npos}{\pop}.
        \]
        Therefore, the counting-based expression of IE satisfies
        \[
        \mathrm{IE}_\gamma(f)
        =
        \frac{T_{\mathrm{model}}}{T_{\mathrm{random}}}
        =
        \frac{
            c'\,\precision
            \;+\;
            \bigl(\budget - c'\bigr)\,
            \dfrac{\npos - c'\,\precision}{\pop - c'}
          }{
            \budget\,\dfrac{\npos}{\pop}
          }.
        \]
        When \(\budget<\tfrac{\npos\,\recall}{\precision}\), this reduces to \(\IEA\); otherwise it reduces to \(\IEB\).

    \item \textbf{Ratio-based expression.}
        Recalling the expression of \(s\) and \(c'\), we can notice that 
        \[
          s = \min\!\Bigl(\gamma, \frac{\pi \recall}{\precision}\Bigr) 
          = \frac{1}{\pop}\cdot \min\!\Bigl(\budget, \frac{\npos \recall}{\precision}\Bigr) 
          = \frac{c'}{\pop}.
        \]
        To leverage this relationship and get rid of \(c'\), \(\npos\) and \(\pop\) in the formula of IE, we divide both the numerator and denominator of the current form with \(\pop\) to obtain the ratio-based expression:
        \begin{align*}
        \mathrm{IE}_\gamma(f) 
        &= \frac{\dfrac{T_{\mathrm{model}}}{\pop}}{\dfrac{T_{\mathrm{random}}}{\pop}} \\[8pt]
        &= \frac{
            \dfrac{c'\,\precision}{\pop}
            +
            \left(\dfrac{\budget}{\pop} - \dfrac{c'}{\pop}\right)\,
            \dfrac{\dfrac{\npos}{\pop} - \dfrac{c'\,\precision}{\pop}}{1 - \dfrac{c'}{\pop}}
          }{
            \dfrac{\budget}{\pop} \cdot \dfrac{\npos}{\pop}
          } \\[10pt]
        &= \frac{
            s \precision + (\gamma - s)\,\dfrac{\pi - s \precision}{1 - s}
          }{\gamma \pi},
        \end{align*}
        whence the stated closed form for \(\mathrm{IE}_\gamma(f)\).
\end{enumerate}
\end{proof}

\section{Theoretical Guarantees of PVF}
\label{theoretical guarantees of pvf}

In this section we provide theoretical guarantees for PVF.  
We formalize how the PVF score behaves as the number of perturbed validation sets, the number of perturbation replicas, and the validation set size grow, and we show that PVF asymptotically selects models that are optimal for a property determined by the perturbation mechanism and the aggregation rule.

\subsection{Setup}

We use the notation introduced in \Cref{sec:pvf}.  
The original validation set is
\[
\Dval = \{(x_i,y_i)\}_{i=1}^{\ntest},
\]
where \(\ntest\) is the number of validation points, each \(x_i\) is a feature vector, and \(y_i\) is its label.

PVF uses a fixed perturbation mechanism that, given a feature vector \(x\), produces a random perturbed feature vector \(X'\).  
We denote this conditional distribution by
\[
X' \mid X = x \sim K(\cdot \mid x).
\]
In the implementation, \(K\) is specified by Gaussian noise for numeric features, label-flip noise for categorical features, and decay rules for ordinal features; for the theory, we only need that \(K\) is fixed and the perturbations are independent across samples and replicas.

From the original data distribution \(\p\) on \((X,Y)\) and the perturbation rule \(K\), we define the induced perturbed distribution \(\tilde{\p}\) on \((X',Y)\) by
\[
(X,Y) \sim \p, \qquad X' \mid X \sim K(\cdot \mid X), \qquad (X',Y) \sim \tilde{\p}.
\]

For each perturbation index \(m = 1,\dots,\npert\), PVF constructs a perturbed validation set
\[
\Dtstm
=
\bigl\{ (x_i^{(m,j)}, y_i)\bigr\}_{i=1,\dots,\ntest;\, j=1,\dots,\nloop},
\]
where, for each fixed \(i\), the replicas \(\{x_i^{(m,j)}\}_{j=1}^{\nloop}\) are i.i.d.\ drawn from \(K(\cdot \mid x_i)\), and the labels \(y_i\) are unchanged.

For a model \(f\) in a finite candidate set \(F\), the evaluation metric on each perturbed set is
\[
P_f^{(m)} = \mathrm{Eval}(f, \Dtstm),
\]
and the vector of perturbed scores is
\[
P_f = \bigl(P_f^{(1)},\dots,P_f^{(\npert)}\bigr) \in \mathbb{R}^{\npert}.
\]
The PVF score of model \(f\) is obtained by an aggregation rule
\[
A_f = A(P_f),
\]
where \(A\colon\mathbb{R}^{\npert}\to\mathbb{R}\) is chosen by the user (for example, \(A\) can be the mean or a lower quantile).  
PVF then selects
\[
\fstar \in \arg\max_{f\in F} A_f.
\]

We make the following assumptions, which will be referenced explicitly in the propositions and proofs.

\begin{enumerate}[label=(A\arabic*)]
    \item \(\Dval = \{(x_i,y_i)\}_{i=1}^{\ntest}\) is drawn i.i.d.\ from a distribution \(\p\) on \((X,Y)\).
    \item For each feature vector \(x\), the perturbation rule is given by a fixed conditional distribution \(K(\cdot \mid x)\), and all perturbations across different samples, replicas, and perturbed sets are independent.
    \item The set of candidate models \(F\) is a finite set.
    \item For any fixed model \(f\) and any dataset \(D\), the evaluation metric \(\mathrm{Eval}(f,D)\) is a bounded function of empirical statistics of \(D\).  
    In particular, for the metrics used in the paper (IE, F1, accuracy), \(\mathrm{Eval}(f,D)\) depends on \(D\) only through the empirical confusion matrix entries, which are empirical averages of bounded per-sample indicators.
    Generally, \(\mathrm{Eval}(f,D)\) should be a metric that calculates a generalized mean of the performance on each data point.
    \item There exists a functional \(\A\) on probability distributions on \(\mathbb{R}\) such that, if \(Z_1,\dots,Z_{\npert}\) are i.i.d.\ with common distribution \(L\), then
    \[
    A(Z_1,\dots,Z_{\npert}) \xrightarrow[\;\npert\to\infty\;]{\text{a.s.}} \A(L).
    \]
    This holds for standard aggregation rules including the sample mean, sample quantiles, the median, and more general continuous \(L\)-statistics.
\end{enumerate}

We now state and prove the theoretical guarantees under these assumptions.

\subsection{Convergence of PVF Scores}

The first step is to understand the behavior of the PVF score \(A_f\) for a fixed model \(f\) as the sampling dimensions \(\npert\), \(\nloop\), and \(\ntest\) grow.

For each fixed \(\ntest,\nloop\) and realization of \(\Dval\), let
\(
\mathcal{L}_{\ntest,\nloop}(f)
\)
denote the conditional distribution of \(P_f^{(1)} = \mathrm{Eval}(f,\Dtst^{(1)})\) given \(\Dval\).  
Similarly, for each \(\ntest\), we will use
\(
\mathcal{L}_{\ntest,\infty}(f)
\)
for the limiting distribution of \(P_f^{(1)}\) as \(\nloop\to\infty\).

\begin{proposition}[i.i.d.\ structure of perturbed scores]
\label{prop:iid}
Suppose (A1) and (A2) hold, and fix \(\Dval\), \(\ntest\), \(\nloop\), and a model \(f\in F\).  
Then, conditional on \(\Dval\), the scores \(P_f^{(1)},\dots,P_f^{(\npert)}\) are i.i.d.\ random variables with common distribution \(\mathcal{L}_{\ntest,\nloop}(f)\).
\end{proposition}

\begin{proof}
Assumption (A2) states that, for each \(i\), all perturbations \(\{x_i^{(m,j)}\}\) across different \(m\) and \(j\) are independent given \(\Dval\) and follow the same distribution \(K(\cdot \mid x_i)\).  
By construction, each \(\Dtstm\) is built from independent perturbations of \(\Dval\), and thus the sets \(\Dtst^{(1)},\dots,\Dtst^{(\npert)}\) are conditionally independent and identically distributed given \(\Dval\).  
Assumption (A4) ensures that $\mathrm{Eval}(f,\cdot)$ is a deterministic and
measurable function of its input dataset.  Because determinism implies that
$\mathrm{Eval}(f,\cdot)$ introduces no additional randomness beyond that already
present in the dataset, applying $\mathrm{Eval}(f,\cdot)$ to each perturbed
dataset $\Dtst^{(m)}$ will not create dependence between the resulting scores.

Formally, for any real intervals $I_1,\dots, I_k$ and any indices
$m_1,\dots,m_k$, the events $\{P_f^{(m_i)} \in I_i\}$ satisfy
\[
\{P_f^{(m_i)} \in I_i\}
=
\{\Dtst^{(m_i)} \in \mathrm{Eval}(f,\cdot)^{-1}(I_i)\},
\]
where $\mathrm{Eval}(f,\cdot)^{-1}(I_i)$ is the preimage of $I_i$ under the
evaluation map.  Since the perturbed datasets
$\Dtst^{(1)},\dots,\Dtst^{(\npert)}$ are independent conditional on $\Dval$
(by construction under Assumption~(A2)), we have
\[
\Pr\!\left(
P_f^{(m_1)} \in I_1,\dots,P_f^{(m_k)} \in I_k \mid \Dval
\right)
=
\]
\[
\prod_{i=1}^k 
\Pr\!\left(
P_f^{(m_i)} \in I_i \mid \Dval
\right).
\]
This shows that evaluating each $\Dtst^{(m)}$ through the deterministic map
$\mathrm{Eval}(f,\cdot)$ preserves the i.i.d.\ structure established at the
level of the perturbed datasets.
 
Therefore, the resulting scores \(P_f^{(1)},\dots,P_f^{(\npert)}\) are conditionally i.i.d.\ given \(\Dval\) with some common distribution, which we denote by \(\mathcal{L}_{\ntest,\nloop}(f)\).
\end{proof}

Interpretation: for a fixed validation set and perturbation parameters, PVF is aggregating many i.i.d.\ draws of a random performance score for each model under the chosen perturbation rule.

\begin{proposition}[Convergence in \(\npert\)]
\label{prop:npert}
Suppose (A1), (A2), (A4), and (A5) hold, and fix \(\Dval\), \(\ntest\), \(\nloop\), and \(f\in F\).  
Then
\[
A_f = A\bigl(P_f^{(1)},\dots,P_f^{(\npert)}\bigr)
\xrightarrow[\;\npert\to\infty\;]{\text{a.s.}}
\A\bigl(\mathcal{L}_{\ntest,\nloop}(f)\bigr).
\]
\end{proposition}

\begin{proof}
By Proposition~\ref{prop:iid}, conditional on \(\Dval\), the scores \(P_f^{(1)},\dots,P_f^{(\npert)}\) are i.i.d.\ with common distribution \(\mathcal{L}_{\ntest,\nloop}(f)\).  
Assumption (A5) states that for any i.i.d.\ sequence from a distribution \(L\), the aggregation rule satisfies
\[
A(Z_1,\dots,Z_{\npert}) \to \A(L) \quad \text{almost surely as } \npert\to\infty.
\]
Applying this with \(Z_m = P_f^{(m)}\) and \(L = \mathcal{L}_{\ntest,\nloop}(f)\) yields the claimed convergence.
\end{proof}

Interpretation: when the validation set and perturbation settings are fixed, increasing the number of perturbed validation sets guarantees that PVF approximates the population-level value of the aggregation rule applied to the distribution of perturbed scores for each model.

\begin{proposition}[Convergence in \(\nloop\)]
\label{prop:nloop}
Suppose (A1), (A2), and (A4) hold, and fix \(\Dval\), \(\ntest\), and \(f\in F\).  
Then, there exists a distribution \(\mathcal{L}_{\ntest,\infty}(f)\) such that
\[
\mathcal{L}_{\ntest,\nloop}(f) \xrightarrow[\;\nloop\to\infty\;]{} \mathcal{L}_{\ntest,\infty}(f),
\]
and, if (A5) also holds,
\[
\A\bigl(\mathcal{L}_{\ntest,\nloop}(f)\bigr)
\xrightarrow[\;\nloop\to\infty\;]{} 
\A\bigl(\mathcal{L}_{\ntest,\infty}(f)\bigr).
\]
\end{proposition}

\begin{proof}
Fix \(\Dval\), \(\ntest\), and \(f\).  
By Assumption (A2), for each original sample \((x_i,y_i)\), the replicas \(\{x_i^{(1,j)}\}_{j=1}^{\nloop}\) are i.i.d.\ from \(K(\cdot\mid x_i)\).  
Assumption (A4) states that \(\mathrm{Eval}(f,\Dtst^{(1)})\) is a bounded function of empirical statistics of \(\Dtst^{(1)}\).  
By the law of large numbers, for the replicas, as \(\nloop\to\infty\), these empirical statistics of \(\Dtst^{(1)}\) converge almost surely to their conditional expectations given \(\Dval\).  
Since \(\mathrm{Eval}(f,\cdot)\) is a deterministic function of these statistics, the distribution of \(P_f^{(1)} = \mathrm{Eval}(f,\Dtst^{(1)})\) also converges to a limiting distribution, which we denote by \(\mathcal{L}_{\ntest,\infty}(f)\).  
Assumption (A5) then implies the convergence of \(\A(\mathcal{L}_{\ntest,\nloop}(f))\) to \(\A(\mathcal{L}_{\ntest,\infty}(f))\).
\end{proof}

Interpretation: as the number of perturbation replicas per validation point grows, the randomness arising from finite perturbation sampling disappears, and the distribution of perturbed scores for each model stabilizes to a limit.

\begin{proposition}[Convergence in \(\ntest\)]
\label{prop:ntest}
Suppose (A1), (A2), and (A4) hold.  
For each model \(f\in F\), define
\[
J(f) := \mathrm{Eval}(f,\tilde{\p}),
\]
the performance value of \(f\) under the perturbed population distribution \(\tilde{\p}\).  
Then, as \(\ntest \to \infty\),
\[
\mathcal{L}_{\ntest,\infty}(f) \xrightarrow[\;\ntest\to\infty\;]{} \delta_{J(f)},
\]
where $\delta_{J(f)}$ is the probability distribution that assigns probability
\(1\) to the single point \(J(f)\); equivalently, a random variable \(Z\) has
distribution $\delta_{J(f)}$ if and only if
\[
\Pr(Z = J(f)) = 1.
\]
If (A5) also holds, then
\[
\A\bigl(\mathcal{L}_{\ntest,\infty}(f)\bigr)
\xrightarrow[\;\ntest\to\infty\;]{P}
\A\bigl(\delta_{J(f)}\bigr)
=: J_A(f).
\]
\end{proposition}

\begin{proof}
Assumption (A1) implies that the empirical distribution of \(\Dval\) converges to \(\p\) as \(\ntest\to\infty\).  
Combined with the perturbation rule in (A2), this implies that the empirical distribution of perturbed samples in each \(\Dtst^{(1)}\) converges to \(\tilde{\p}\).  
By Assumption (A4), \(\mathrm{Eval}(f,\Dtst^{(1)})\) is a bounded function of empirical statistics of \(\Dtst^{(1)}\).  
These statistics converge to their expectations under \(\tilde{\p}\), and hence \(\mathrm{Eval}(f,\Dtst^{(1)})\) converges in probability to \(\mathrm{Eval}(f,\tilde{\p}) = J(f)\).  
Therefore, the distribution \(\mathcal{L}_{\ntest,\infty}(f)\) converges to \(\delta_{J(f)}\).  
Assumption (A5) then yields \(\A(\mathcal{L}_{\ntest,\infty}(f)) \to \A(\delta_{J(f)})\).
\end{proof}

Interpretation: with a sufficiently large validation set, all randomness in the perturbed score for a fixed model collapses to a single deterministic value \(J(f)\), and the limit of any reasonable aggregation of that score is just a deterministic function \(J_A(f)\) of the same quantity.

\begin{proposition}[Unified convergence of PVF scores]
\label{prop:score-convergence}
Suppose (A1)–(A5) hold.  
Then, for any fixed model \(f \in F\),
\[
A_f \xrightarrow[\;\npert,\nloop,\ntest\to\infty\;]{P} J_A(f),
\]
where \(J_A(f) = \A(\delta_{J(f)})\) and \(J(f) = \mathrm{Eval}(f,\tilde{\p})\).
\end{proposition}

\begin{proof}
For fixed \(\Dval\), \(\ntest\), and \(\nloop\), Proposition~\ref{prop:npert} implies
\[
A_f \xrightarrow[\;\npert\to\infty\;]{\text{a.s.}} \A\bigl(\mathcal{L}_{\ntest,\nloop}(f)\bigr).
\]
Proposition~\ref{prop:nloop} implies
\[
\A\bigl(\mathcal{L}_{\ntest,\nloop}(f)\bigr)
\xrightarrow[\;\nloop\to\infty\;]{} \A\bigl(\mathcal{L}_{\ntest,\infty}(f)\bigr),
\]
and Proposition~\ref{prop:ntest} implies
\[
\A\bigl(\mathcal{L}_{\ntest,\infty}(f)\bigr)
\xrightarrow[\;\ntest\to\infty\;]{P}
\A\bigl(\delta_{J(f)}\bigr)
=: J_A(f).
\]
Combining these convergences (almost sure in \(\npert\), then in \(\nloop\), then in probability in \(\ntest\)) yields the stated convergence of \(A_f\) to \(J_A(f)\) in probability as all three variables grow.
\end{proof}

Interpretation: for each model, as we use more perturbed validation sets, more perturbation replicas, and more validation data, the PVF score converges to a deterministic limit that depends only on the underlying data distribution, the perturbation rule, the evaluation metric, and the aggregation rule.

\subsection{Consistency of PVF Model Selection}

We now show that, under the same assumptions, PVF consistently selects models that maximize the objective \(J_A(f)\).

Define
\[
J_A \colon F \to \mathbb{R},
\qquad
f \mapsto J_A(f),
\]
and the set of maximizers
\[
F_A^\star := \{ f \in F : J_A(f) = \max_{g\in F} J_A(g) \}.
\]

\begin{proposition}[Consistency of PVF selection]
\label{prop:selection}
Suppose (A1)–(A5) hold.  
Then the PVF-selected model \(\fstar\) satisfies
\[
\fstar \xrightarrow[\;\npert,\nloop,\ntest\to\infty\;]{P} f_\infty^\star,
\]
for some \(f_\infty^\star \in F_A^\star\).
\end{proposition}

\begin{proof}
By Proposition~\ref{prop:score-convergence}, for each fixed \(f\in F\),
\[
A_f \xrightarrow[\;\npert,\nloop,\ntest\to\infty\;]{P} J_A(f).
\]
Assumption (A3) states that \(F\) is finite. Let \(F=\{f_1,\dots,f_Q\}\).
For each \(i\in\{1,\dots,Q\}\), define
\[
A_i := A_{f_i},
\qquad
J_i := J_A(f_i).
\]
Since convergence in probability holds for each coordinate individually, the whole score vector converges in probability:
\[
(A_1,\dots,A_Q)
\xrightarrow[\;\npert,\nloop,\ntest\to\infty\;]{P}
(J_1,\dots,J_Q).
\]

To analyze the argmax selection rule, let
\[
F_A^\star := \bigl\{ i: J_i = \max_{j=1,\dots,Q} J_j \bigr\}
\]
be the (nonempty) set of population maximizers.  
For any index \(i\notin F_A^\star\), we have a strict gap
\[
J_{i^\star} - J_i > 0
\qquad\text{for every } i^\star\in F_A^\star .
\]
Define
\[
\Delta := \min_{i\notin F_A^\star}\, (J_{i^\star} - J_i),
\]
which is strictly positive. 

Now, fix any \(\varepsilon < \Delta/3\).  
We can see that, if all coordinates satisfy
\[
|A_i - J_i| < \varepsilon \qquad\text{for all } i=1,\dots,Q,
\]
then, for every \(i^\star\in F_A^\star\) and \(i\notin F_A^\star\),
\[
A_{i^\star} \ge J_{i^\star} - \varepsilon
> J_i + \Delta - \varepsilon
> J_i + 2\varepsilon
> A_i.
\]
Hence, whenever the coordinatewise errors are smaller than \(\varepsilon\), every truly optimal model has strictly larger estimated score than every non-optimal model, and thus the maximizer of the estimated scores must lie in the true maximizer set, ensuring that we get the true optimal model:
\[
\arg\max_{i=1,\dots,Q} A_i \subseteq F_A^\star.
\]

Since
\[
\Pr\!\left(
\max_i |A_i - J_i| > \varepsilon
\right)
\longrightarrow 0,
\]
it follows that
\[
\Pr\!\left(
\arg\max_{i} A_i \subseteq F_A^\star
\right)
\longrightarrow 1.
\]
Any deterministic tie-breaking rule selects a single element of \(\arg\max_i A_i\), so the selected model index converges in probability to a member of \(F_A^\star\).  
Rewriting indices back in terms of models, we obtain:
\[
\fstar \in \arg\max_{f\in F} A_f
\quad\Longrightarrow\quad
\fstar \xrightarrow{P} F_A^\star.
\]

Thus any PVF selection rule converges in probability to a population maximizer of \(J_A\).
\end{proof}

Interpretation: PVF is a consistent model selection procedure.  
In the limit of large validation sets and sufficient perturbation sampling, PVF will choose a model whose PVF score is maximal among all candidates.

\subsection{PVF as Selection Toward a User-Specified Property}

Finally, we interpret the limit objective \(J_A\) as an induced property determined by the design choices in PVF.

For a fixed perturbation rule \(K\), evaluation metric \(\mathrm{Eval}\), and aggregation rule \(A\), we define the induced property
\[
\Phi_{A,K}(f) := J_A(f) = \A\bigl(\delta_{J(f)}\bigr),
\qquad
J(f) = \mathrm{Eval}(f,\tilde{\p}).
\]
This property assigns to each model \(f\) a scalar that reflects its performance under the perturbed distribution \(\tilde{\p}\), summarized through the aggregation rule \(A\).

\begin{proposition}[PVF selects toward the induced property]
\label{prop:property}
Suppose (A1)–(A5) hold.  
Let \(\Phi_{A,K}(f) := J_A(f)\) be the induced property described above.  
Then
\[
\fstar
\xrightarrow[\;\npert,\nloop,\ntest\to\infty\;]{P}
f_{\mathrm{opt}},
\]
for some
\[
f_{\mathrm{opt}} \in \arg\max_{f\in F} \Phi_{A,K}(f).
\]
\end{proposition}

\begin{proof}
This is a direct restatement of Proposition~\ref{prop:selection}, since \(\Phi_{A,K}(f)\) is defined to be \(J_A(f)\) for each \(f\in F\).
\end{proof}

Interpretation: once a user fixes a perturbation mechanism and an aggregation rule, these choices implicitly define a property \(\Phi_{A,K}\) of interest (for example, robustness to specific perturbations, emphasis on lower-tail performance, or any other stability notion encoded by \(K\) and \(A\)).  
The propositions in this section show that, under the stated assumptions, PVF is guaranteed to select models that maximize this induced property.

\section{Detailed Setups and Results of Synthetic Dataset Experiments}
\label{sec:synthetic-setups-results}

In this section, we present the detailed experimental setup for the synthetic dataset experiment, along with analysis about its results under proper hyperparameter tuning.

\subsection{Synthetic Dataset Experimental Setup}
\label{sec:synthetic-setup}

\paragraph{Data generation}
To simulate key attributes of clinical prediction tasks, 
we take into consideration class imbalance, limited sample size, measurement noise, and a preference for interpretable decision rules. Specifically, we create two classes in $\mathbb{R}^2$ (with independent Gaussian features $(x_1,x_2)$) that determine the true signal and three irrelevant independent Gaussian features $(x_3,x_4,x_5)$ that disturb the feature space. 
Ground-truth 2D cluster 1 is centered at $(0,0)$ and cluster 2 is at $(\mu,\mu)$ with $\mu \in \{0.1, 0.3, \ldots, 2.9\}$ to span low-through-high separability.
Class imbalance is fixed at 80\% negatives and 20\% positives. 
We study two dataset sizes, $n \in \{50, 100\}$, to reflect data-scarce settings.

\paragraph{Validation protocols}
Upon generating one synthetic dataset, 
it is split into 70\% for training and 30\% for validation using stratification to preserve class balance. 
We compare two model/feature selection strategies: PVF and the traditional method, based on whether they select the ground-truth feature pair $(x_1,x_2)$.

To mitigate imbalance in training folds, 
we apply SMOTE oversampling (and minority duplication when the minority count is below two) before fitting models on the training set.

\paragraph{Candidate models}
Clinical settings prioritize interpretability, 
often favoring simpler models like logistic regression or decision trees \cite{abdullah2021review, rudin2019stop, Christodoulou2019}. 
In practice, one can even decide the complexity of the model based on heuristic judgement on overfitting and underfitting. 
In this case, choosing the best model often reduces to choosing the right subset of features. 
This equivalence is clearest for bounded-depth decision trees or sparse logistic regression,  
whose predictive power hinges directly on the selected variables \cite{breiman2001statistical, chen2022tree, ning2022interpretable, tan2022figs, nguyen2023interpretable, suttaket2024rational}. 
Following these ideas, we use logistic regression models trained on pairs of features only. 
This yields ten pairs and thus ten model candidates, 
one per distinct pair among the five features.
By the Rashomon Effect,
there might be more than one pair of features that achieve optimal or almost-optimal performances.
For example, although the truly informative pair is $(x_1,x_2)$, 
due to a small number of samples, 
$(x_1,x_3)$ might accidentally performs equally well or even better than $(x_1,x_2)$. 
Our design deliberately creates the possibility of encountering this situation under the Rashomon Effect,
and aligns feature selection with model selection: 
the best model should be the one trained on the truly informative pair $(x_1,x_2)$. 
It creates a clear, ground-truth notion of “correct selection” while keeping the interpretability of the selected model.

\paragraph{Model evaluation metrics}
The primary analysis uses Intervention Efficiency (IE) at different intervention capacities $\gamma \in \{0.1, 0.3, 0.5, 0.7, 0.9\}$. 
For the  complementary analysis under traditional metrics, 
we report F1 Scores when class separability is relatively low ($\mu \leq 2.5$) and accuracies when class separability is very high ($\mu > 2.5$).
This intentionally favors the traditional method, 
because based on our empirical results, 
using F1 Score when $\mu \leq 2.5$ and accuracy when $\mu > 2.5$ makes the traditional method more likely to select the best model (truly informative feature pair). 
Therefore, we force PVF to adopt the same setting to stress-test it and avoid underestimating the traditional method by the "wrong" metric.
Selection is always compared within a fixed metric: 
PVF versus traditional under IE, or PVF versus traditional under F1 Score/accuracy.

\paragraph{Hyperparameters, settings and repetitions}
We vary the dataset size ($n = 50$ or $100$) and the separability parameter $\mu$, 
which takes values uniformly spaced between 0.1 and 2.9 with a step size of 0.2,
and perturbation Gaussian noise standard deviation $\sigma \in \{10^{-6}, 10^{-5}, 10^{-4}, 10^{-3}, 10^{-2}, 10^{-1}\}$ as different perturbation controls. 
For each hyperparameter configuration ($2 \times 15 \times 6 = 180$ combinations) in \Cref{tab:hyperparameters}, 
we generate 5000 independent synthetic datasets to obtain stable comparison results and uncertainty summaries. 
For other hyperparameters discussed in \Cref{pvf-table} while not in \Cref{tab:hyperparameters}, 
we empirically set $d = 5$, $k = 7$ and $M = 100$ to balance the computational cost and perturbation effect.
The aggregation function $\mathcal A$ of PVF is manually chosen to be the 25th percentile (Q1) of a model's validation scores across perturbed validation sets for demonstration,
as it is well suited to high-stakes clinical contexts where worst-case generalization matters and they could correspond to the cases where all independent noises compound in a way that is most disturbing to the model's prediction, while not being too conservative.
Note that the choice of aggregation function in other applications could be decided according to specific user requirements of the attributes of the selected model,
as \Cref{prop:property} has stated that ideally PVF should be able to select the optimal model toward any user-defined property.

\begin{table*}[t]
\centering
\caption{Summary of hyperparameters and settings in the synthetic dataset experiment}
\label{tab:hyperparameters}
\begin{tabular}{l l}
\hline
\textbf{Hyperparameters} & \textbf{Values} \\
\hline
Original Data Size ($n$) & 50, 100 \\
Center of Cluster 2 ($\mu$) & From 0.1 to 2.9 with step 0.2 \\
Perturbation Noise Standard Deviation ($\sigma$) & 1e-6, 1e-5, 1e-4, 1e-3, 0.01, 0.1 \\
\hline
\end{tabular}
\end{table*}

\paragraph{Selection method assessment}
For each synthetic dataset, we record the feature pair chosen by PVF and the traditional method. Specifically, the pair of features among the total 10 pairs that achieves the highest logistic regression performance on the perturbed validation sets/the original validation set is selected as the preferred pair for the corresponding method and synthetic dataset. 
The quality of a selection method is assessed by the frequency with which a method chooses the truly informative pair $(x_1,x_2)$ out of 5000 repetitions. 
Specifically, for each hyperparameter configuration $h = (n, \mu, \sigma)$,the performance difference between the PVF and the traditional method is computed as
\[
d_h = c^{\text{PVF}}_h - c^{\text{Trad}}_h,
\]
where $c^{\text{PVF}}_h$ and $c^{\text{Trad}}_h$ are the counts of selecting the truly informative pair $(x_1,x_2)$ for configuration $h$.

\paragraph{Compute environment and runtime}
Synthetic sweeps were executed on a SLURM-managed computing cluster. 
Each job was allocated 11 CPU cores and 30 GB RAM within a fixed Python environment. 
Cluster nodes are equipped with 94-core x86-64 CPUs and about 740 GB RAM, running Linux kernel 5.14.0. 
The size-50 sweep (i.e., evaluating every hyperparameter configuration for data size 50) required 1642 minutes, and the size-100 sweep required 1598 minutes. 
Data generation time was negligible relative to model fitting.
Scaling by cores, repetitions (5000), and hyperparameter settings (180),
the inner-loop PVF selection for one synthetic dataset and one hyperparameter on a single CPU is about 1.2 seconds for all model evaluation metrics.

\subsection{Synthetic Dataset Experimental Results}\label{sec:synthetic-results}
In this subsection, 
we present the experimental results in the form of a comparison between the performance of PVF with tuned $\sigma$ and that of the traditional method.
Across all hyperparameter configurations, 
statistical significance is confirmed using paired $t$–tests. 
During the experiments, 
we found that a key factor influencing the effectiveness of the proposed method is class separability, 
represented by the parameter $\mu$. 
Hence, we categorized $\mu$ into three separability scenarios: 
low (0.1 – 0.9), moderate (1.1 – 1.9), and high (2.1 – 2.9). 
We aggregate the differences $d_h$'s within each class separability range to get the mean difference for the corresponding range.

\paragraph{Data size 50} 
Across all class separability, 
PVF always outperforms the traditional method no matter which model evaluation metric is used.
The performance difference is overall more significant for lower $\gamma$'s.

Specifically, 
when the class separability is low, 
PVF overall selects the truly informative feature pair 42 to 65 more times than the traditional method does, 
across all metrics (\Cref{fig:best-50-1}).
This consistently positive difference confirms the superiority of PVF over the traditional method in low-separability cases, 
regardless of the metric.

\begin{figure}[tbp]
    \centering
    \includegraphics[width=0.47\textwidth]{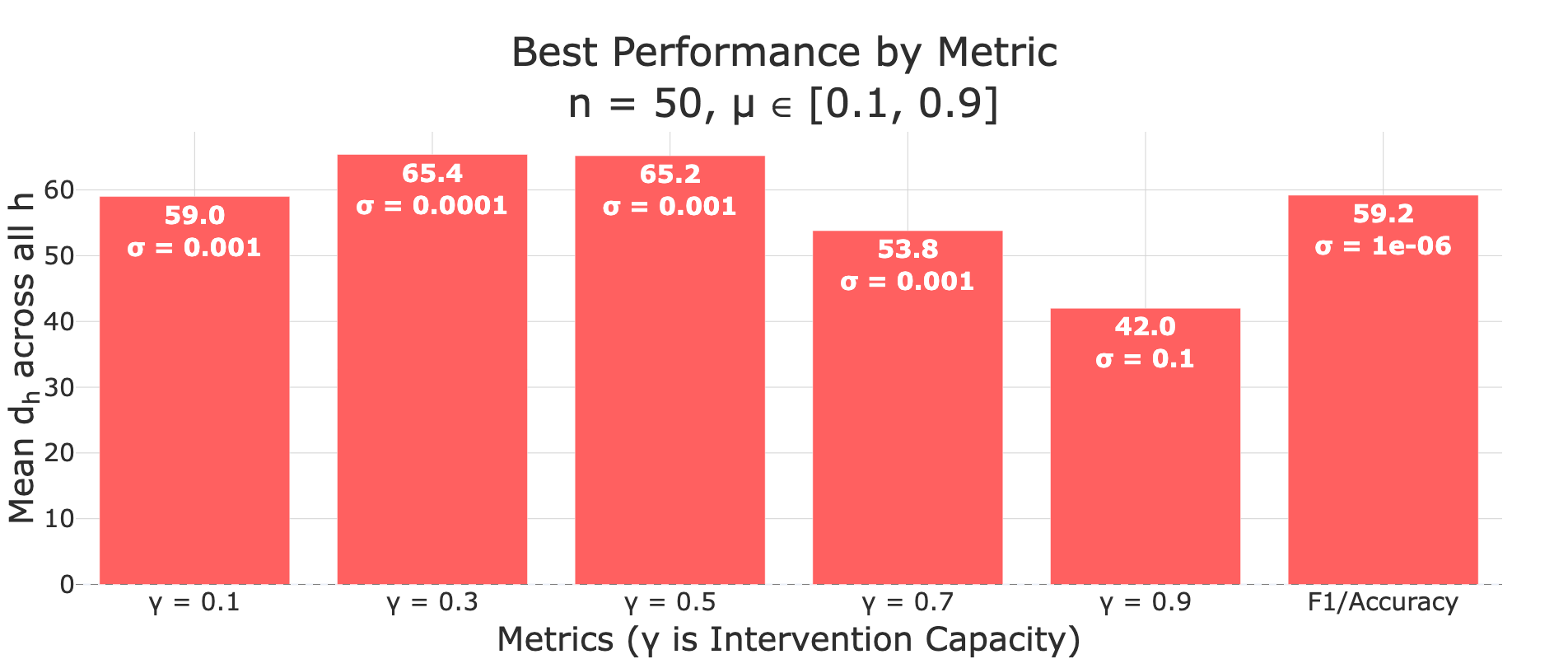}
    \caption{The best performance of PVF on synthetic datasets ($n = 50, \mu \in [0.1, 0.9]$) under different  metrics, when $\sigma$ is tuned. PVF identifies the truly informative feature pair more frequently across all evaluation metrics. 
}
    \label{fig:best-50-1}
\end{figure}

When the class separability is moderate, 
built on consistent positive differences, 
the value of difference fluctuates more,
with larger values emerging at $\gamma = 0.1, 0.5, 0.7$ (\Cref{fig:best-50-2}). 
This could be due to a higher variability in synthetic datasets and the resulting complex relationship between IE and the classification task. 
F1 score/accuracy exhibit substantial positive differences, 
confirming that PVF is promising when these two traditional metrics are of interest in moderate-separability cases.

When the class separability is high, 
we can witness a clearer pattern in the relationship between the performance difference and $\gamma$. 
That is, when $\gamma$ increases, 
the positive difference shrinks almost monotonically (\Cref{fig:best-50-3}), 
with smaller $\gamma$'s (0.1, 0.3 and 0.5) enjoying a rising positive difference while larger $\gamma$'s (0.7, 0.9) suffering from a loss in positive difference. 
PVF with F1 Score/accuracy still maintains a large edge in the positive performance. 
These observations confirms that under large separability circumstances,
not only is PVF consistently better than the traditional method, 
but also it brings large improvements when the metric of interest is IE with lower $\gamma$'s or F1 Score/accuracy.

\paragraph{Data size 100} 
Although generally PVF still outperforms the traditional method for most metrics when data size is 100, 
a small number of cases emerge where the traditional method becomes preferable.
Specifically, with moderate separability and $\gamma = 0.7$ or $0.9$,
the traditional method selects the true feature pair a small number of times ($<$20) more than PVF (\Cref{fig:best-100-2}).

Nevertheless, in all other cases, PVF consistently outperforms the traditional method (\Cref{fig:best-100-1,fig:best-100-3}).
Besides the general trend that the positive edge is larger at smaller $\gamma$'s as in data size 50,
the strongest comparison occurs when the separability is high and the metrics of interest are IE with $\gamma$ (0.1, 0.3 and 0.5) or F1 Score/accuracy (\Cref{fig:best-100-3}). 
Overall, 
despite a small number of slightly underperforming cases, 
PVF under most metrics enjoys a large edge over the traditional method with data size 100.

\begin{figure}[tbp]
    \centering
    \includegraphics[width=0.47\textwidth]{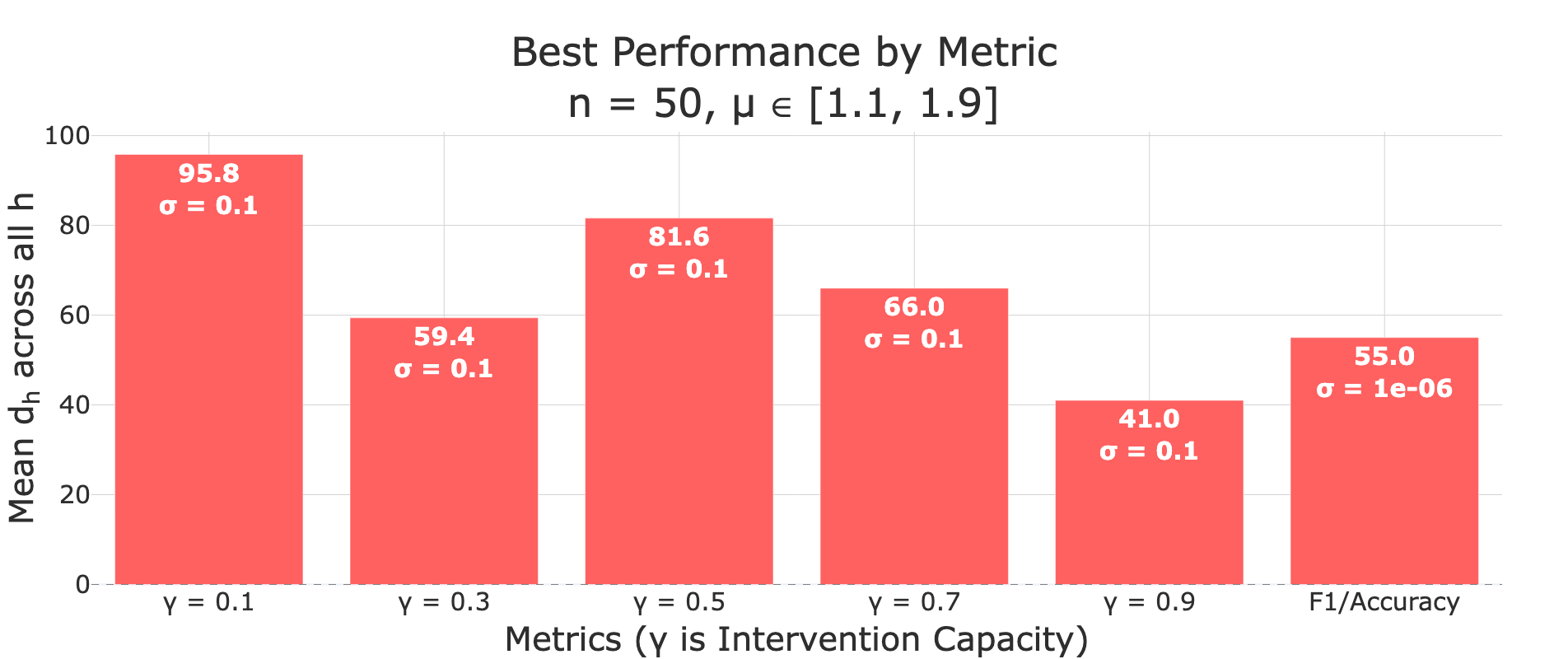}
    \caption{The best performance of PVF on synthetic datasets ($n = 50, \mu \in [1.1, 1.9]$) under different  metrics, when $\sigma$ is tuned. PVF continues to outperform the traditional method across all metrics, with notably higher improvements when $\gamma \in \{0.1, 0.5, 0.7\}$. Variability increases relative to low-separability conditions, reflecting more complex interactions between IE and classification difficulty.}
    \label{fig:best-50-2}
\end{figure}

\begin{figure}[tbp]
    \centering
    \includegraphics[width=0.47\textwidth]{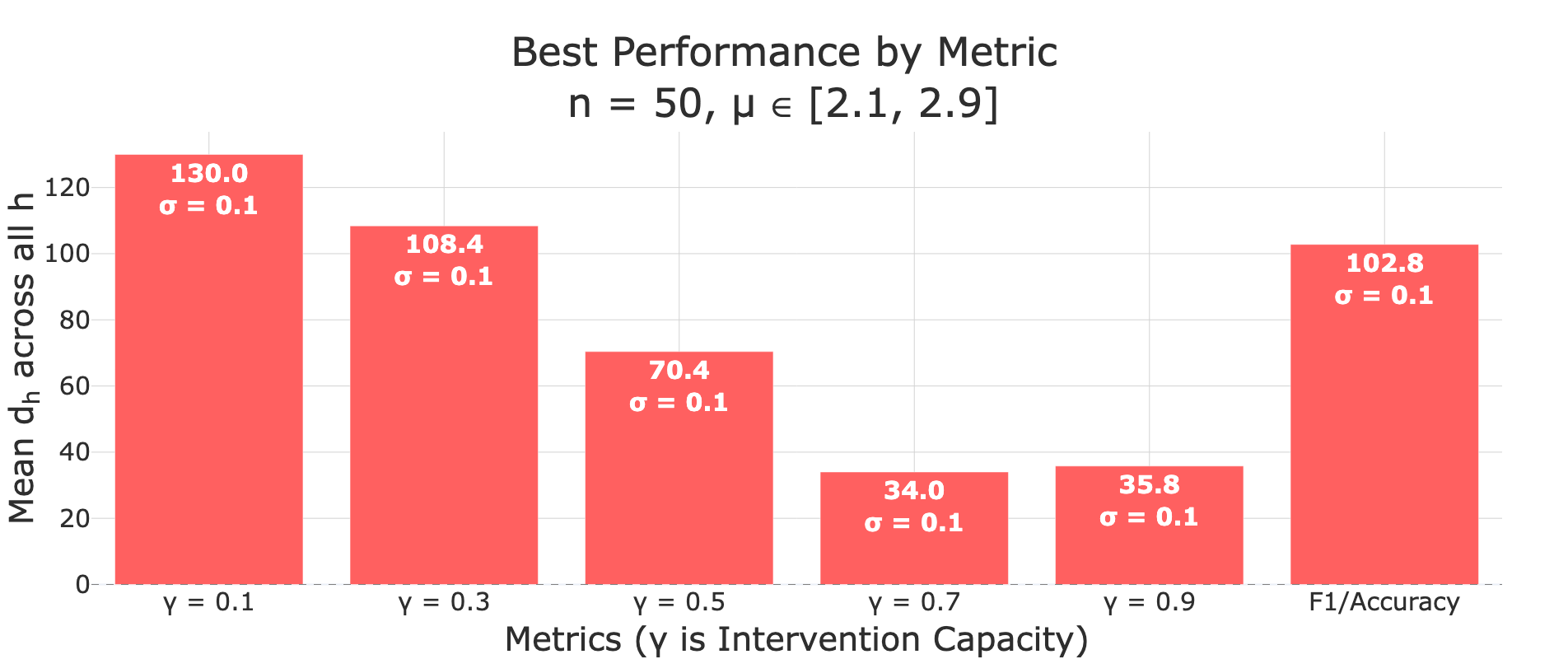}
    \caption{The best performance of PVF on synthetic datasets ($n = 50, \mu \in [2.1, 2.9]$) under different  metrics, when $\sigma$ is tuned. PVF still outperforms the traditional method, but the magnitude of improvement declines almost monotonically as $\gamma$ increases.}
    \label{fig:best-50-3}
\end{figure}

\begin{figure}[tbp]
    \centering
    \includegraphics[width=0.47\textwidth]{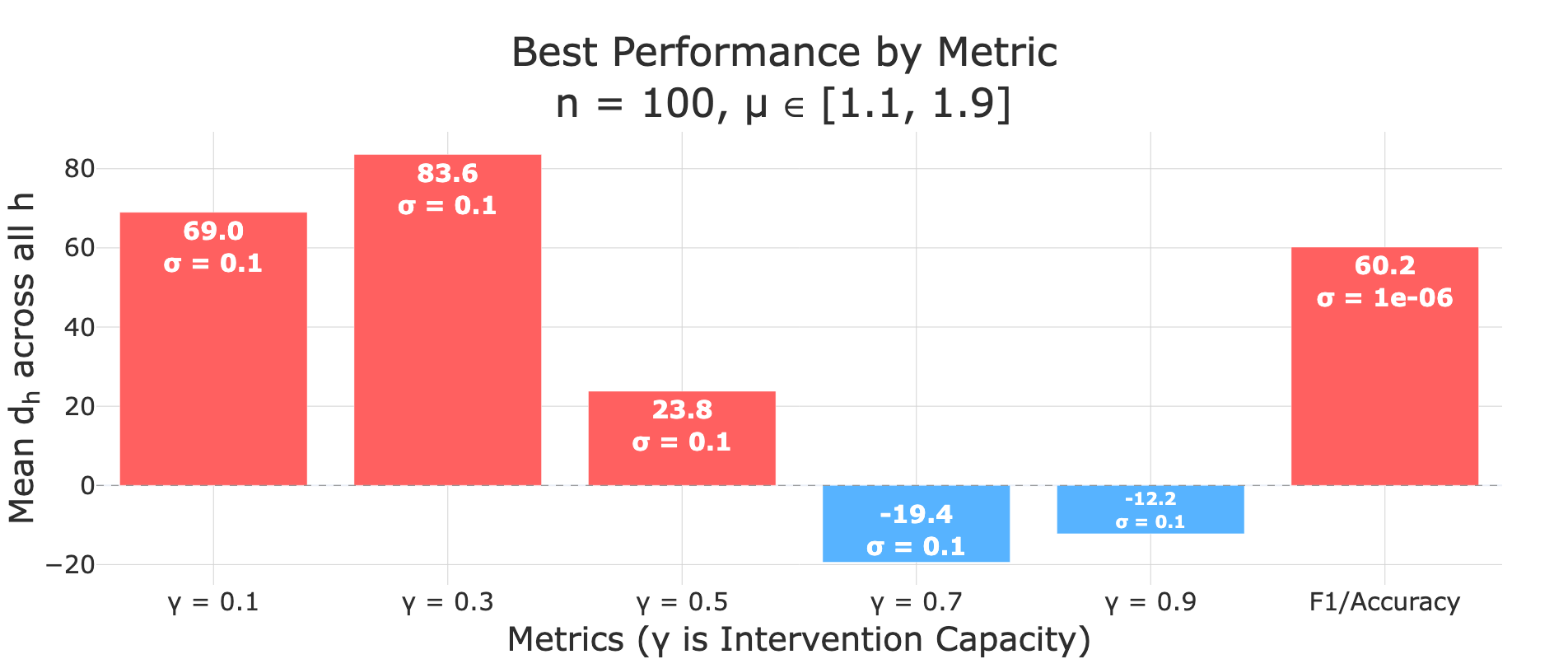}
    \caption{The best performance of PVF on synthetic datasets ($n = 100, \mu \in [1.1, 1.9]$) under different metrics, when $\sigma$ is tuned. Although PVF generally outperforms the traditional method, a few cases where the traditional method selects the true feature pair slightly more often appear when $\gamma \in \{0.7, 0.9\}$ }
    \label{fig:best-100-2}
\end{figure}

\begin{figure}[tbp]
    \centering
    \includegraphics[width=0.47\textwidth]{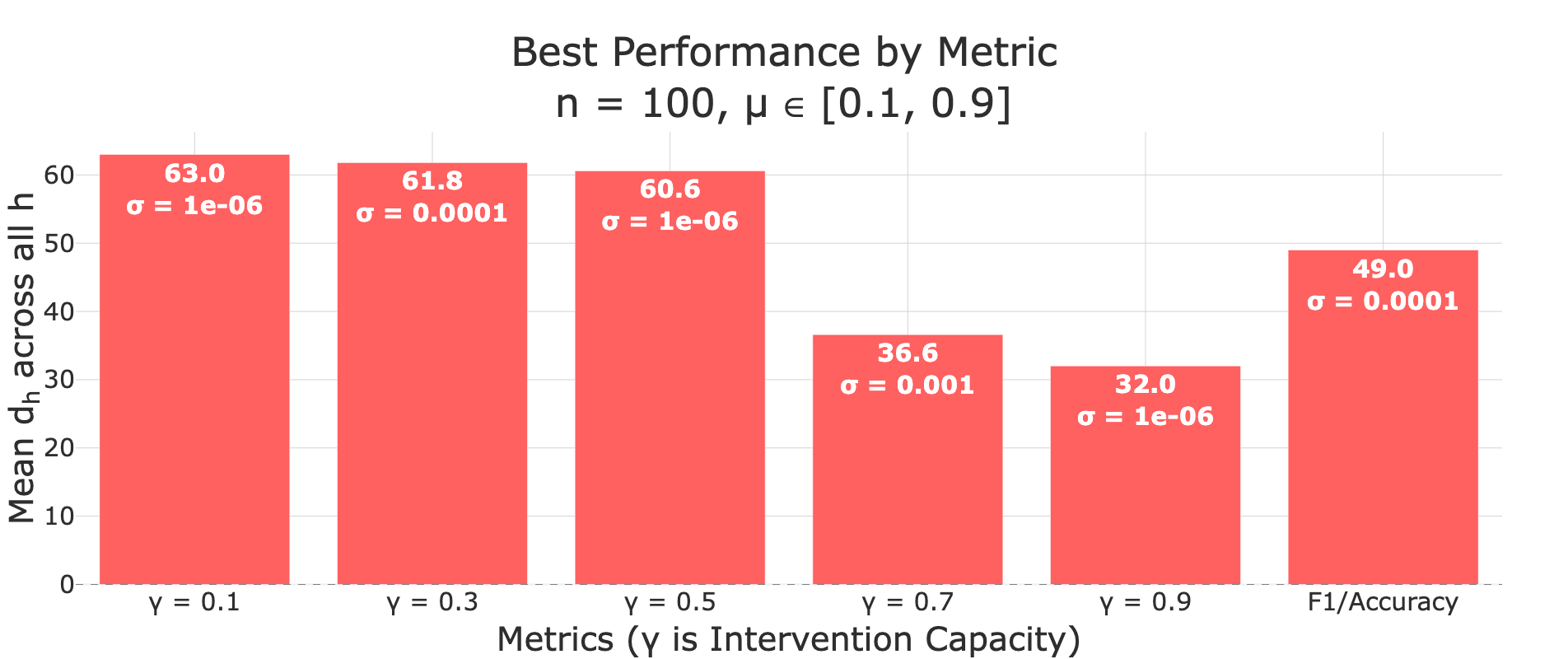}
    \caption{The best performance of PVF on synthetic datasets ($n = 100, \mu \in [0.1, 0.9]$) under different metrics, when $\sigma$ is tuned. PVF consistently selects the informative feature pair more frequently across all metrics, with larger advantages at smaller $\gamma$ values.}
    \label{fig:best-100-1}
\end{figure}

\begin{figure}[tbp]
    \centering
    \includegraphics[width=0.47\textwidth]{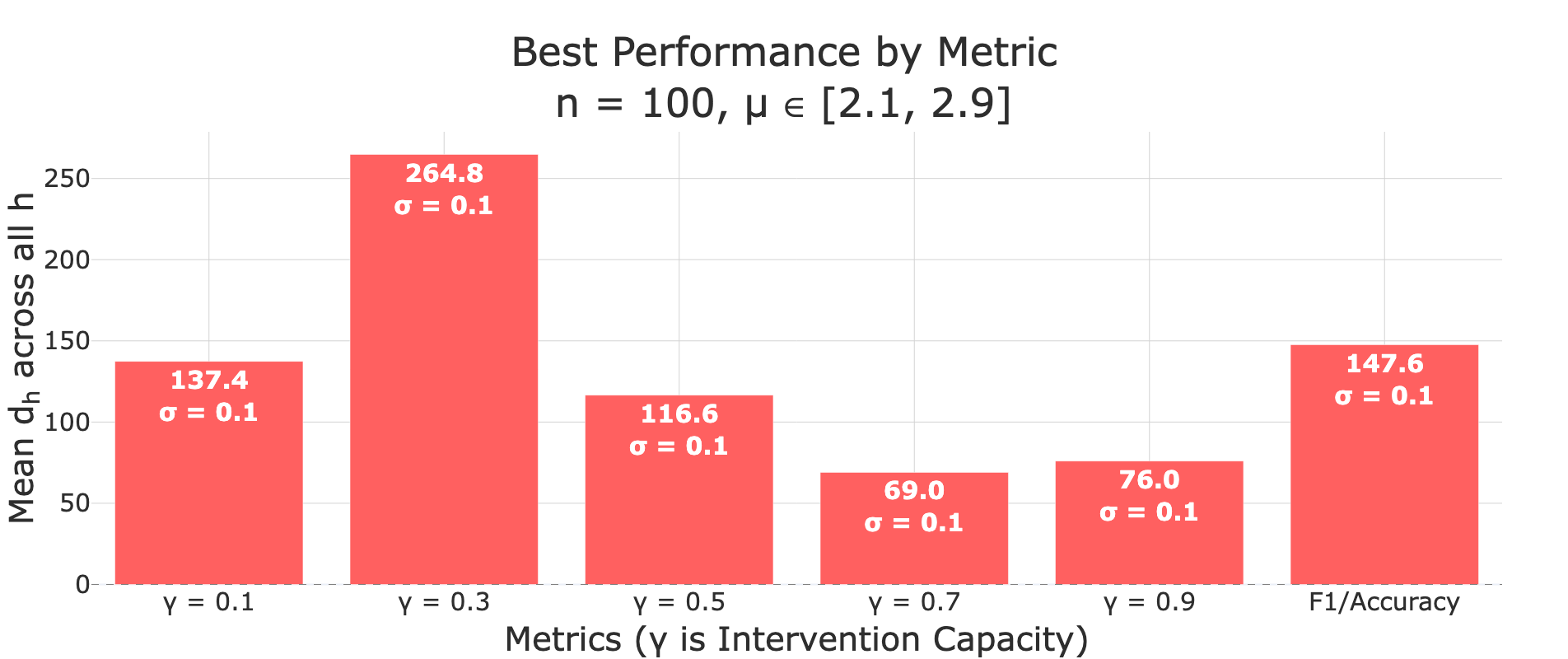}
    \caption{The best performance of PVF on synthetic datasets ($n = 100, \mu \in [2.1, 2.9]$) under different metrics, when $\sigma$ is tuned. PVF exhibits its strongest improvements under high separability, especially when $\gamma \in \{0.1, 0.3, 0.5\}$ and for F1 Score/accuracy.}
    \label{fig:best-100-3}
\end{figure}

\section{Detailed Setups of Real Data Experiments} 
\label{sec:real-setup}
In this section, we present the detailed experimental setup for the real dataset experiments.

\paragraph{Datasets and preprocessing}
We evaluate on two public clinical datasets to ground the study in realistic settings. 
The Breast Cancer Wisconsin (Diagnostic) dataset \cite{wolberg1993breast} has 569 samples, 212 positives, and no missing values; 
no preprocessing beyond standard scaling is needed.
The Cervical Cancer (Risk Factors) dataset \cite{fernandes2017cervical} has 808 samples and substantial missing values. 
We discard features with more than 50\% missing values and impute the remaining entries with $k$-nearest neighbors ($k=5$). 
This yields 668 usable samples with 34 features. 
Among the four diagnostic targets, we use the Schiller test outcome because its higher prevalence provides adequate positive cases for stable validation.

\paragraph{Validation protocols, evaluation metrics, hyperparameters and settings}
To emulate small-cohort deployments and create multiple independent model selection tasks, 
we partition each dataset into several disjoint subsets, each containing 100 samples, and perform model selection within each subset. 
For the Breast Cancer dataset ($n = 569$), 
five non-overlapping 100-sample subsets are formed from 500 samples, 
while the remaining 69 samples are reserved as part of the external test data. 
For a given subset, 
the external test set consists of all remaining samples not included in that subset, 
i.e., $n - 100 = 469$ instances for the Breast Cancer dataset. 
Within each 100-sample subset, 
we further divide the data into five different train–validation splits, 
each time using a distinct split for validation while keeping the rest for training. 
Repeating this procedure five times reduces the influence of random partitioning within the subset and enables a more reliable assessment of whether one selection method consistently identifies better models across multiple validation scenarios and possibilities. 
This design focuses on comparing the stability and consistency of model selection performance, 
rather than evaluating the final predictive accuracy of any individual model.

The same configuration is applied to the Cervical Cancer dataset ($n = 668$), 
resulting in six disjoint 100-sample subsets. 
All train–validation divisions are stratified to preserve class ratios. 
We apply standard scaling to the training data and use the same transformation on validation data to (i) maintain comparability with the synthetic experiments and (ii) make the perturbation noise scale in PVF interpretable across datasets.

We compare PVF (using the 25th percentile aggregator) with the traditional method under the same evaluation metrics and hyperparameters as in the synthetic experiments, except for the perturbation controls. 
For the perturbation controls, 
we extend the previous range of ${noise}\_{add}\_{std}$ 
($\{10^{-6}, 10^{-5}, 10^{-4}, 10^{-3}, 10^{-2}, 10^{-1}\}$) 
to include $\{2 \times 10^{-1}, 3 \times 10^{-1}, 4 \times 10^{-1}, 5 \times 10^{-1}\}$ 
to further examine robustness gains under stronger perturbations. 
Additionally, for ordinal and nominal features, 
we empirically set the feature-changing probability to 0.1 
and the decay parameter $\lambda$ for ordinal features to 0.1 
(see \Cref{sec:pvf} for details on $\lambda$), 
leaving the exploration of alternative choices to future work.

\paragraph{Candidate models}
We use decision trees as the base model class to prioritize interpretability and to account for potential non-linear separability. 
Within each 100-sample subset, 
we generate a diverse pool of candidate models by repeatedly subsampling 70\% of the training data and training one tree per subsample (maximum depth 4 for interpretability), 
yielding 100 candidates in total. 
This resampling procedure intentionally produces multiple models with similar overall accuracy but differing feature splits and decision paths, 
characterizing the presence of the the Rashomon Effect. 

\paragraph{Selection method assessment}
For each validation split of each data subset and each model selection method, we pick one model from the candidate pool using the chosen metric based on validation performance. The selected model is then evaluated on the external test set associated with this piece using the same metric as in validation. This aligns validation with testing and approximates ground truth generalization. We repeat the procedure across subsets and splits, as explained in the last paragraph, and aggregate the testing performance w.r.t the model evaluation metrics, subsets and splits.

\paragraph{Compute environment and runtime}
Real-data experiments run locally on an Apple MacBook Pro (M4 Pro, 12-core CPU, 16-core GPU, 24 GB RAM) in a fixed Python environment. The two datasets are processed concurrently with a total time of 424.3 minutes.

\section{Sensitivity Analysis}\label{sec:sensitivity-analysis}
In this section, we discuss how sensitive is PVF to the problem setting and the key hyperparameter in PVF: perturbation control, i.e., the standard deviation of the Gaussian perturbation noise, $\sigma$.

\subsubsection{Synthetic Dataset Experiment Sensitivity Analysis}
In this subsection, 
we isolate the effect of the perturbation noise standard deviation $\sigma$ by holding other settings fixed and tracking the mean selection difference within each class separability case. 

\paragraph{Data size 50}
We start with data size 50. 
First, when class separability is low, 
very small perturbations ($\sigma \le 10^{-3}$) maintain a stable, positive advantage for PVF (\Cref{fig:line-50-1}). 
In this range the perturbed validation replicates the original validation set closely while still eliminating some extent of occasionality of a single split.
Increasing $\sigma$ beyond this range can reduce the fidelity of the perturbed sets to the underlying problem, 
compressing score differences across model candidates and generally weakening the advantage, except for larger $\gamma$'s (0.7 and 0.9). 
The drop is pronounced at $\sigma = 10^{-2}$ and makes the difference  negative at $\sigma = 10^{-1}$ for F1 Score.

Second, at moderate class separability, 
performance difference often slightly decreases from a stable positive value when $\sigma$ goes from a small value towards a moderate one (\Cref{fig:line-50-2}),
while the difference increases back to a higher value when $\sigma$ reaches the largest value in the setting ($\sigma=0.1$). 
This could be because: 
a small perturbation already has positive effects on selecting the best model because it stress-test the candidates; 
a slightly stronger perturbation enables reshuffling the rankings of the candidate models to a greater extent but not yet strong enough to consistently separate stable from unstable candidates, 
while the disturbance has exceeded the benefits brought by the stress test; 
at the upper end of the perturbation range, 
the stress test becomes the strongest,
and somehow the advantage recovers.

Third, at high class separability, 
larger perturbations amplify PVF’s advantage almost certainly for all metrics (with a small fluctuation in the curve for $\gamma = 0.9$), as shown in \Cref{fig:line-50-3}. 
With a large perturbation, 
well-supported candidate models remain top of the rank across perturbed validation sets, 
while models whose strength is fraudulently exaggerated by some noisy patterns due to small sample size fall in rank easily.

\paragraph{Data size 100}
With $n=100$ and low separability, the small-$\sigma$ region is not merely flat but consistently above zero (\Cref{fig:line-100-1}), 
indicating a robust positive margin for PVF across all very small perturbations. 
The guidance here is to keep $\sigma$ not too large,
and every value in this wide range would be very positively effective. 
In contrast, when $\sigma$ becomes too large under low separability, 
the same degradation appears as in $n=50$. 

In moderate-separability settings, 
for IE at more stringent capacities ($\gamma=0.1,0.3$), PVF’s positive performance difference margin is relatively stable at small perturbations ($\sigma\le 10^{-3}$), dips around $\sigma = 10^{-2}$, and then increases sharply at $\sigma=10^{-1}$ (\Cref{fig:line-100-2}). This agrees with the results from $n=50$ that a weak stress test is already informative; while a mid-range $\sigma$ can reshuffle the rankings of the candidate models to a larger extent without consistently separating stable candidates from unstable ones; and a larger $\sigma$ provides a decisive stress that favors candidates whose rankings remain intact under huge perturbation.

For IE at $\gamma=0.5$, the advantage is near zero or slightly negative at very small $\sigma$, then becomes positive only at the largest $\sigma$ ($10^{-1}$), indicating that stronger stress is required before PVF separates candidates meaningfully in this setting. 

For IE at $\gamma=0.7$ and $0.9$, 
PVF is negative everywhere to a small extent, 
especially at smaller $\sigma$'s. 
IE at $\gamma=0.9$ drops further at $\sigma=10^{-2}$, 
with a large recovery at $\sigma=10^{-1}$, 
while IE at $\gamma=0.7$ simply recovers at $\sigma=10^{-3}$. 
Here, when the intervention capacity is large, 
small perturbations add variation without enough discriminatory power over models.
Under F1, 
PVF’s margin slightly decrease with increasing $\sigma$, 
while remaining above zero. 
This indicates that PVF is still beneficial, 
and bad performances under IE at $\gamma=0.7$ and $0.9$ may simply due to IE's lack of capability of recognizing truly good models in large-capacity scenarios.

In high-separability settings, 
increasing $\sigma$ raises the curves for both $n=100$ and $n=50$, 
with $n=100$ exhibiting especially large gains at the top of the perturbation range \Cref{fig:line-100-3}. 
This is consistent with the results from $n=50$ that higher separability enables the stress test from PVF to be more informative and effective.

\paragraph{Summary of sensitivity analysis for synthetic dataset experiments} 
Overall, for IE, 
small perturbations are preferred when class separability is low, 
and large perturbations become beneficial as separability increases. 
Also, IE with higher $\gamma$'s generally enjoys less benefit from PVF than the IE with lower $\gamma$'s does. 
For F1 Score, 
except for high-separability cases, 
keeping $\sigma$ small always gives stably good results. 
Only when the separability is high does using large $\sigma$ gives better results, 
with a steep improvement as for IE.

\begin{figure}[tbp]
    \centering
    \includegraphics[width=0.47\textwidth]{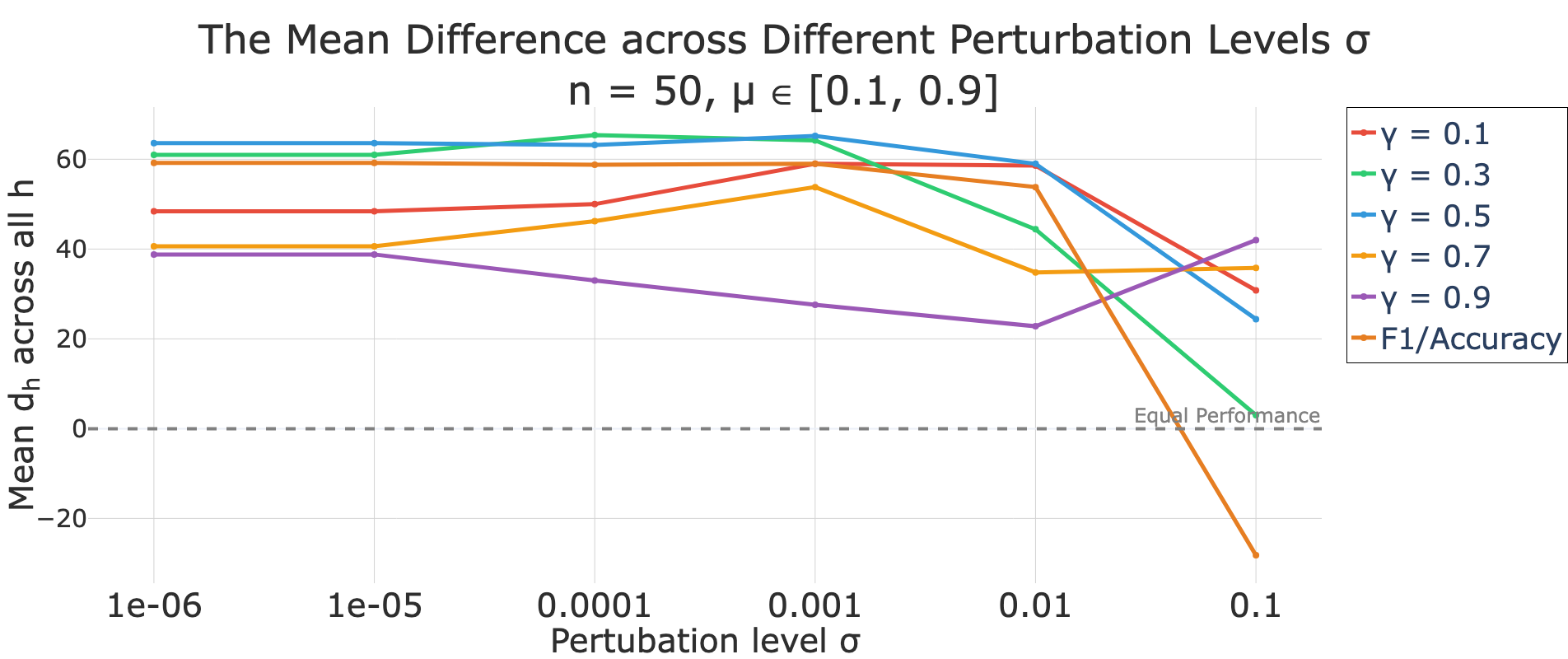}
    \caption{The performance difference between PVF and the traditional method on synthetic datasets ($n = 50, \mu \in [0.1, 0.9]$) across perturbation levels $\sigma$ and evaluation metrics. PVF maintains a stable positive advantage at relatively small $\sigma$ ($\sigma \leq 10^{-2}$), while larger perturbations slightly reduce its margin and can overturn it for F1 Score/accuracy.}
    \label{fig:line-50-1}
\end{figure}

\begin{figure}[tbp]
    \centering
    \includegraphics[width=0.47\textwidth]{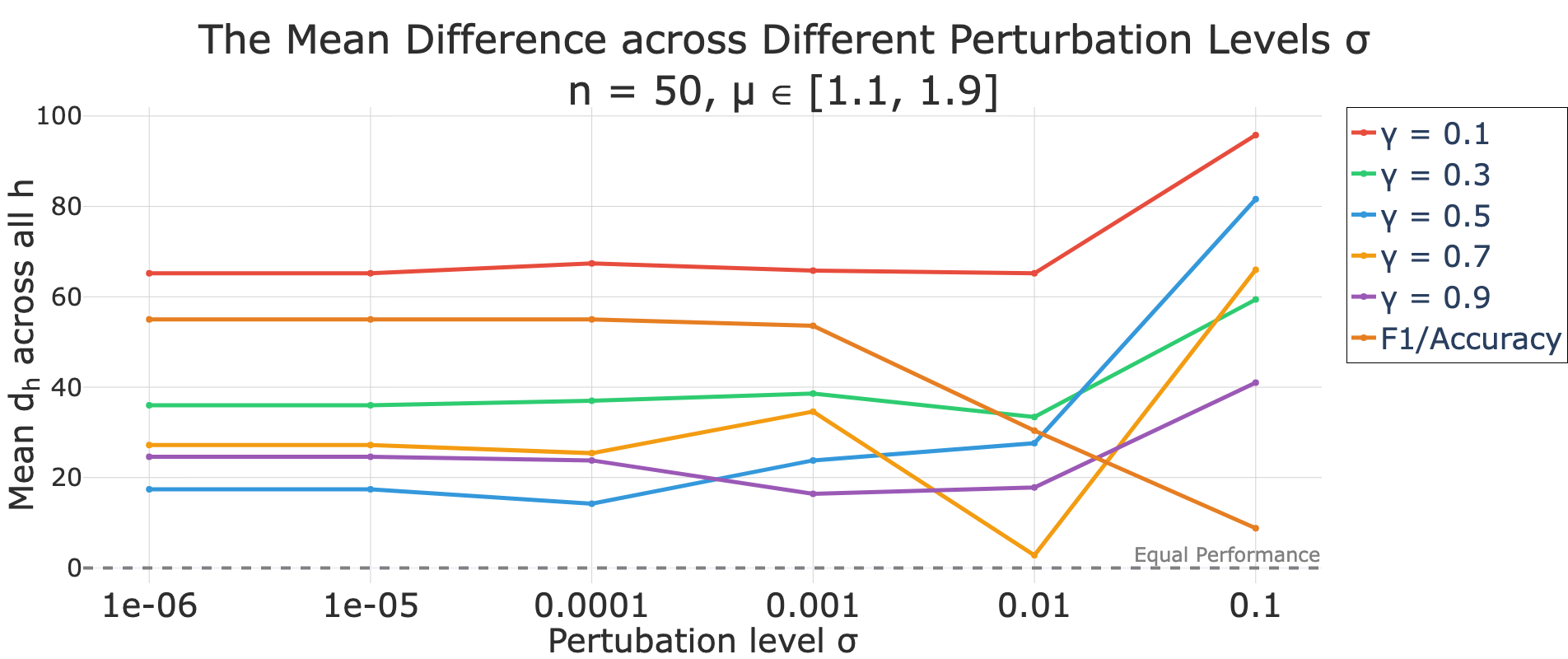}
    \caption{The performance difference between PVF and the traditional method on synthetic datasets ($n = 50, \mu \in [1.1, 1.9]$) across perturbation levels $\sigma$ and evaluation metrics. The advantage of PVF dips as $\sigma$ increases from small to moderate values, but recovers at the largest perturbation level ($\sigma = 0.1$).}
    \label{fig:line-50-2}
\end{figure}

\begin{figure}[tbp]
    \centering
    \includegraphics[width=0.47\textwidth]{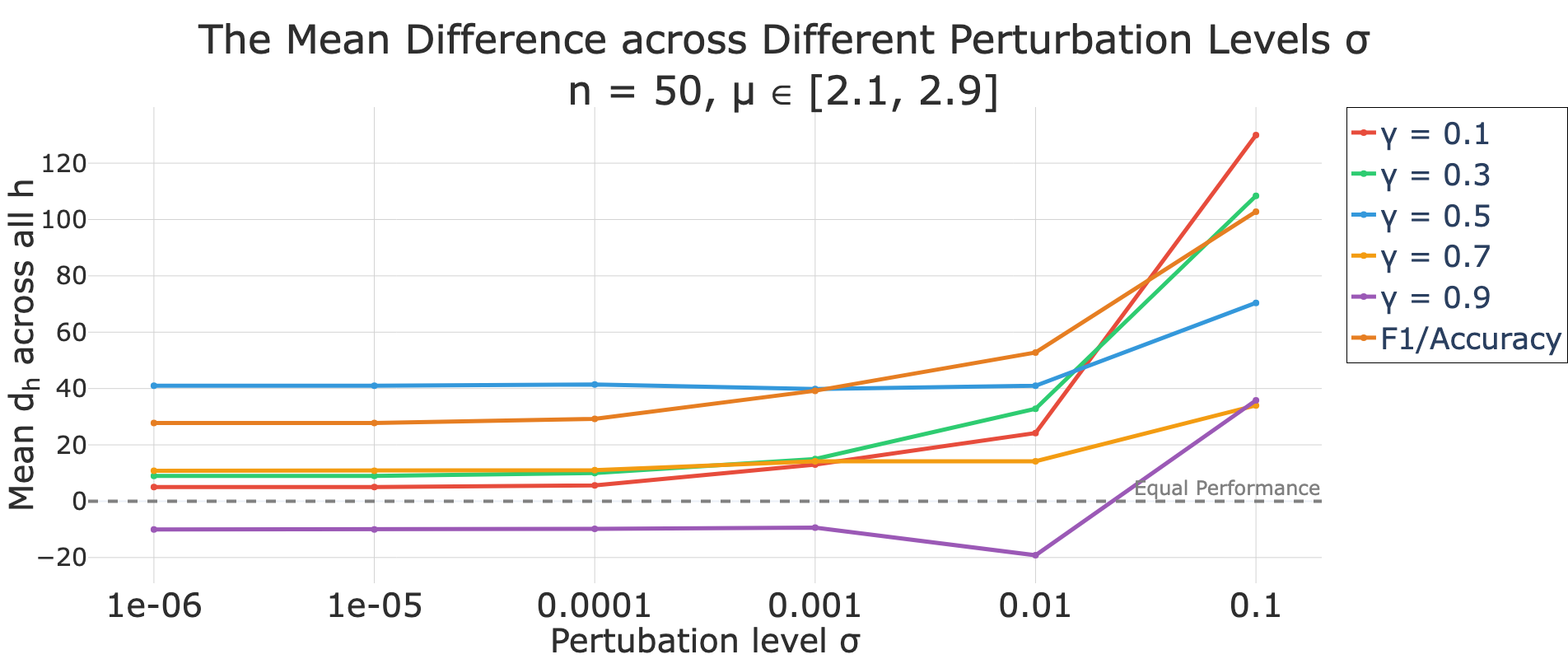}
    \caption{The performance difference between PVF and the traditional method on synthetic datasets ($n = 50, \mu \in [2.1, 2.9]$) across perturbation levels $\sigma$ and metrics. Larger perturbations consistently amplify PVF’s advantage, with only minor fluctuations when $\gamma = 0.9$.}
    \label{fig:line-50-3}
\end{figure}

\begin{figure}[tbp]
    \centering
    \includegraphics[width=0.47\textwidth]{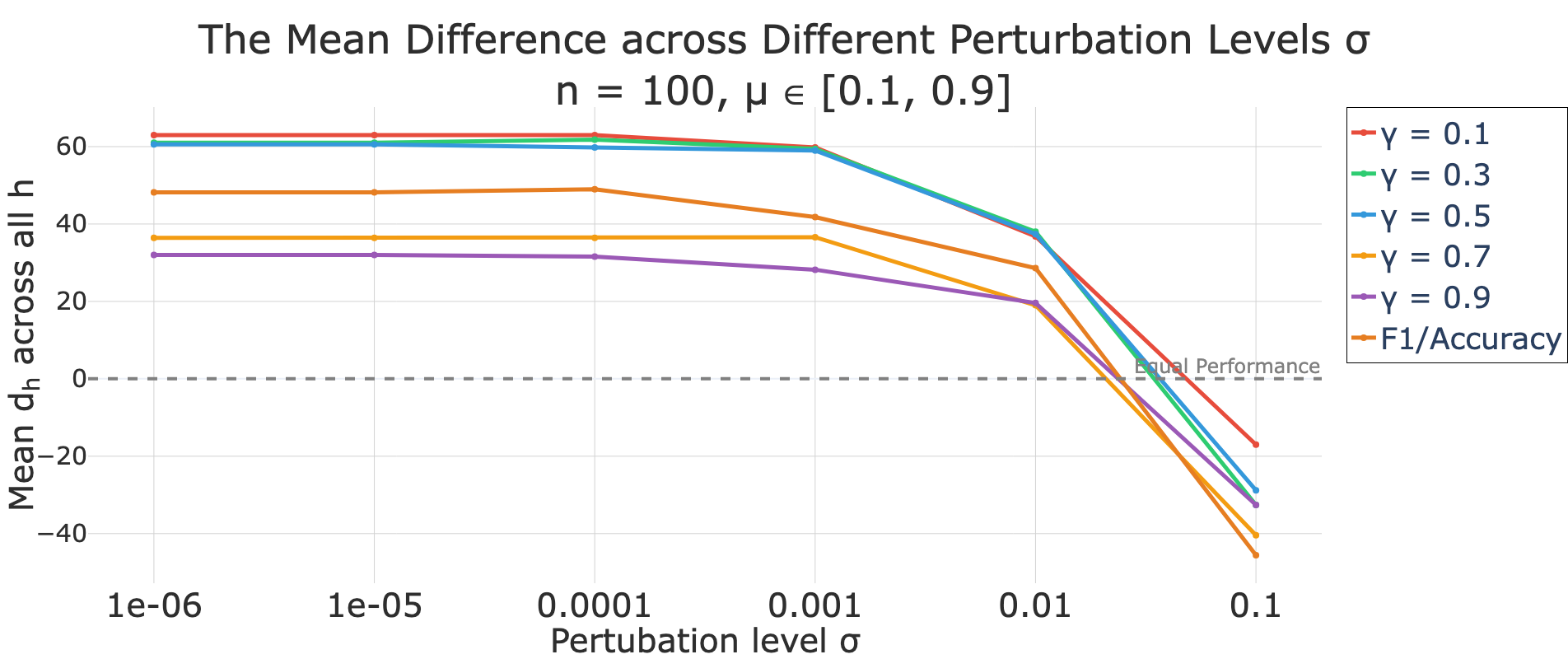}
    \caption{The performance difference between PVF and the traditional method on synthetic datasets ($n = 100, \mu \in [0.1, 0.9]$) across perturbation levels $\sigma$ and metrics. PVF maintains a positive margin over small to moderate $\sigma$, while the largest large perturbation reverses the advantage, similarly to the $n = 50$ setting.}
    \label{fig:line-100-1}
\end{figure}

\begin{figure}[tbp]
    \centering
    \includegraphics[width=0.47\textwidth]{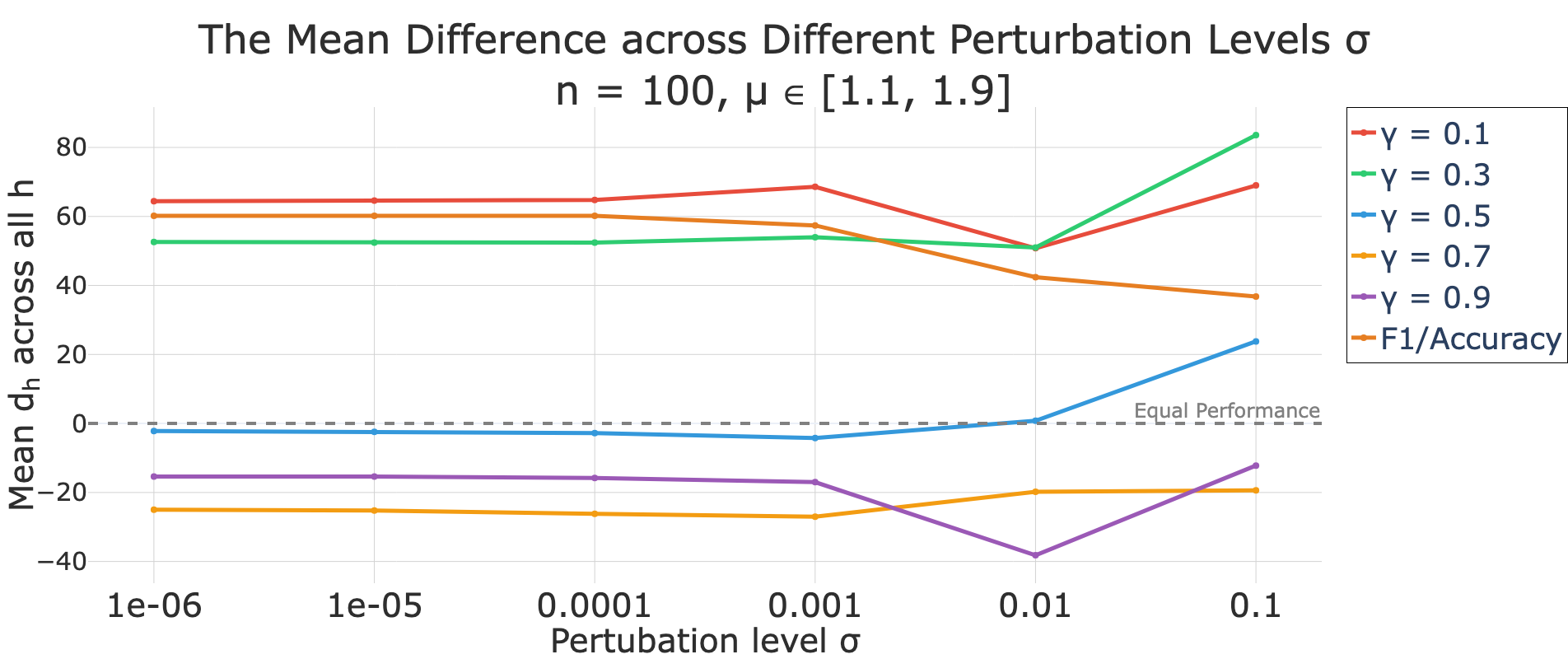}
    \caption{The performance difference between PVF and the traditional method on synthetic datasets ($n = 100, \mu \in [1.1, 1.9]$) across perturbation levels $\sigma$ and metrics. PVF shows stable gains at small $\sigma$, dips at mid-range $\sigma$, and recovers at the largest perturbation level, with the pattern varying by evaluation metrics.}
    \label{fig:line-100-2}
\end{figure}

\begin{figure}[tbp]
    \centering
    \includegraphics[width=0.47\textwidth]{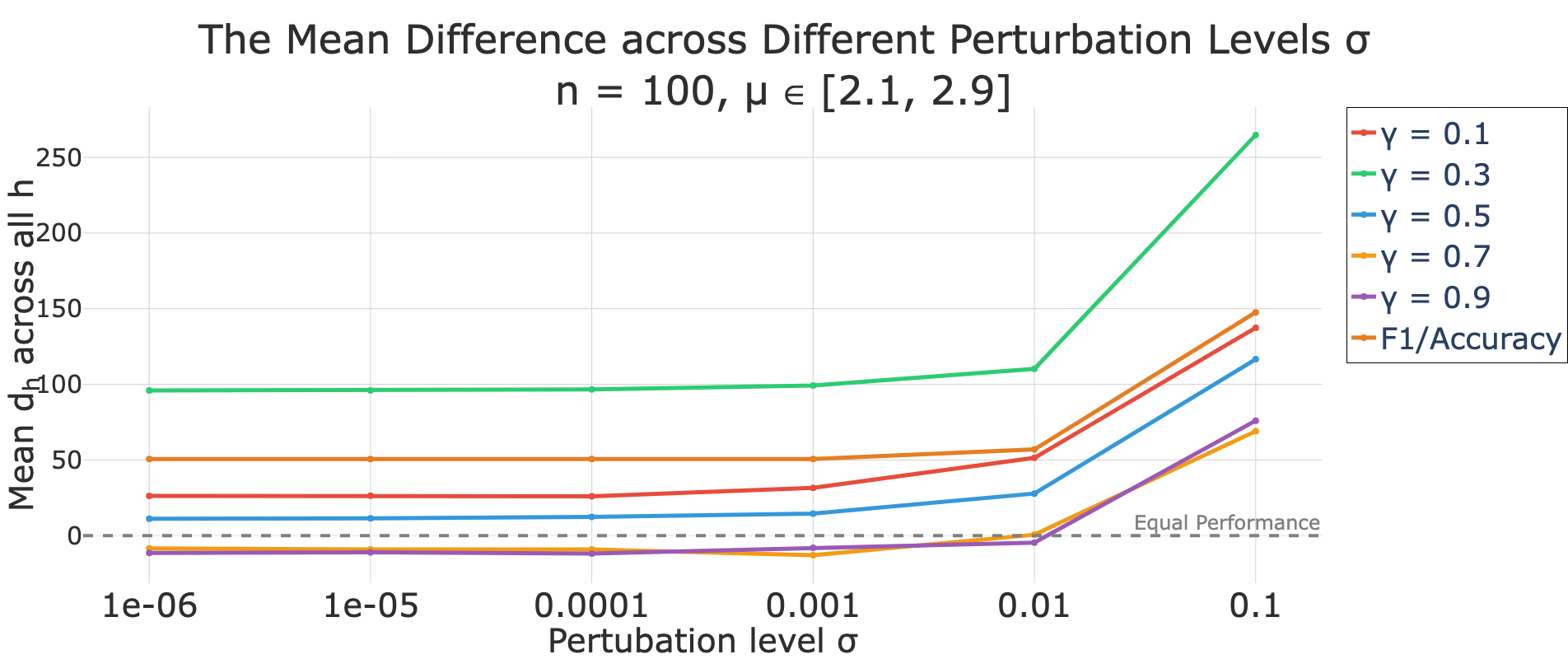}
    \caption{The performance difference between PVF and the traditional method on synthetic datasets ($n = 100, \mu \in [2.1, 2.9]$) across perturbation levels $\sigma$ and metrics. Increasing $\sigma$ steadily amplifies PVF’s advantage, producing especially strong improvements at the largest perturbation level.}
    \label{fig:line-100-3}
\end{figure}

\subsubsection{Real Dataset Experiment Sensitivity Analysis}
In this subsection, 
like in the synthetic sensitivity analysis subsection,
we isolate the effect of the perturbation noise standard deviation $\sigma$ by looking at the results of two datasets separately and examining PVF's performance under all possible choices of $\sigma$. 
The performance difference is defined as the proportion of repetitive experiments in which PVF selects a superior model based on the external test set minus the proportion in which the traditional method does,
across all subsets and splits.

\paragraph{Cervical cancer}
\Cref{fig:cervical-line} shows that the perturbation level $\sigma$ yielding the largest advantage for PVF depends on the evaluation setting. 
For the smallest capacity ($\gamma=0.1$), 
PVF’s margin is positive across the grid, 
increases with the rise of $\sigma$, 
and peaks at $\sigma = 0.01$, 
after which it declines as $\sigma$ further increases. 
For broader capacities ($\gamma\ge 0.3$), 
the difference is small, 
none or even turns negative once $\sigma\ge 0.1$, 
indicating that large perturbations blur the ranking that IE induces and no longer help PVF distinguish candidates. 
For F1, the curve is flat and positive at very small $\sigma$ and drops to zero near $\sigma=0.1$, 
with a small recovery around $\sigma=0.2$ and another dropp to zero if $\sigma$ further increase.
Overall, the most reliable gains for this dataset occur at small $\sigma$ (near $10^{-2}$), 
while larger perturbations reduce PVF’s benefit.

\begin{figure}[tbp]
    \centering
    \includegraphics[width=0.47\textwidth]{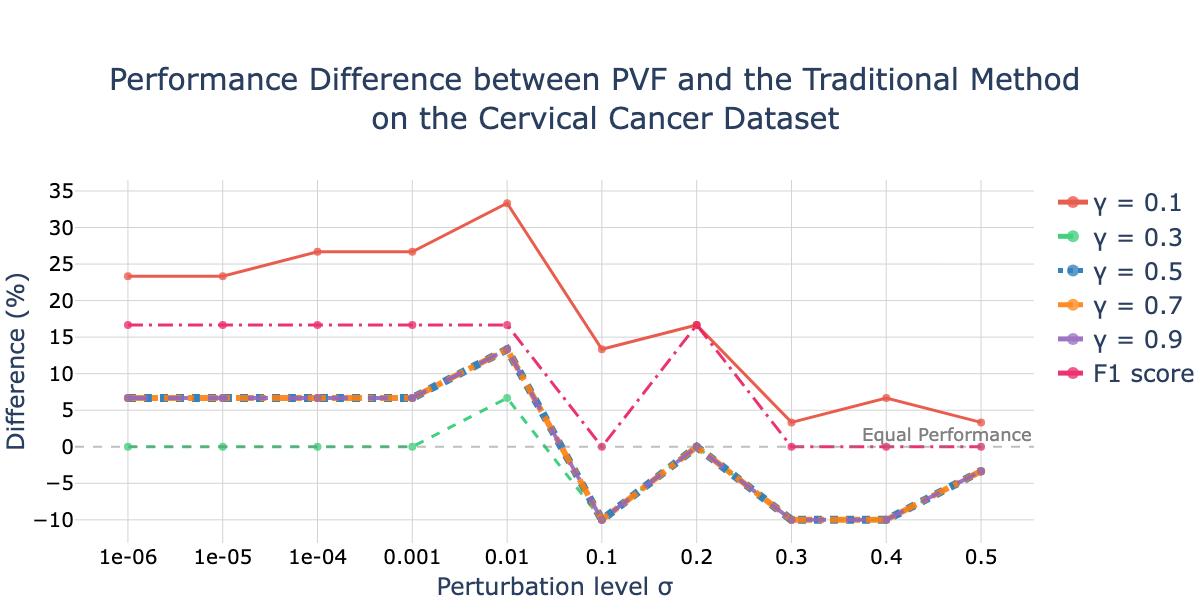}
    \caption{The performance difference between PVF and the traditional method on the cervical cancer dataset across perturbation levels $\sigma$ and evaluation metrics. PVF achieves the most reliable improvements at small $\sigma$ values (around $10^{-2}$), while larger perturbations diminish or reverse its advantage, especially when $\gamma \geq
 0.3$.}
    \label{fig:cervical-line}
\end{figure}

\paragraph{Breast cancer}
\Cref{fig:breast-line} exhibits a different profile. 
Very small perturbations ($\sigma\le 10^{-4}$) yield ties across settings.
As $\sigma$ moves into the $10^{-3}$–$10^{-1}$ range, 
PVF’s advantage increases, 
especially for $\gamma\in\{0.1,0.3\}$ and F1 Score. 
For IE under small $\gamma$ (0.1 and 0.3) and F1 Score, 
the largest differences all occur at $\gamma = 0.3$. 
Beyond 0.3, the curves either flatten or decrease slightly. 
For larger capacities ($\gamma\ge 0.5$), 
the advantage of PVF remains near zero at moderate $\sigma$ and becomes negative only at the largest $\sigma$ values. 
Thus, for this dataset, 
moderate to large perturbation depending on the metric of interest is required for PVF to reveal stable candidates. 
Too little perturbation produces ties, 
while excessive perturbation can erode the margin in the large-capacity settings.

\begin{figure}[tbp]
    \centering
    \includegraphics[width=0.47\textwidth]{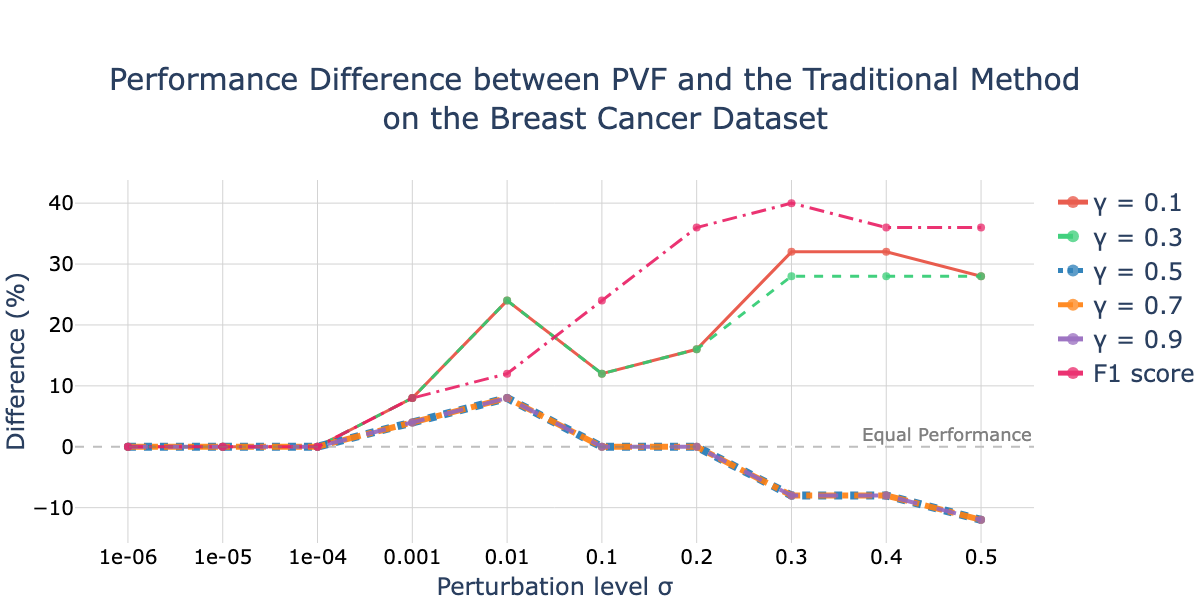}
    \caption{The performance difference between PVF and the traditional method on the breast cancer dataset across perturbation levels $\sigma$ and evaluation metrics. PVF benefits most from moderate perturbations ($\sigma \in [10^{-3},10^{-1}]$), particularly for $\gamma \in \{0.1, 0.3\}$ and F1 Score, whereas very small or very large $\sigma$ values can lead to ties or reduced gains depending on the metric.}
    \label{fig:breast-line}
\end{figure}

\paragraph{Summary of sensitivity analysis for real dataset experiments} 
The two datasets demand different perturbation scales: 
small $\sigma$ is most effective for cervical cancer (with a clear optimum near $10^{-2}$ at $\gamma=0.1$), 
whereas breast cancer benefits from large $\sigma$'s (peaking near $0.3$–$0.4$ for smaller capacities and F1 Score) or moderate $\sigma$'s (peaking near $10^{-2}$ for larger capacities), 
depending on the metric of interest.
These patterns underscore that $\sigma$ should be tuned per dataset and evaluation setting rather than being kept as a fixed prior.
With that being said, empirically, 
we can see that a moderate perturbation (e.g., $\gamma=10^{-2}$) generally applies well to both datasets, 
which could be a starting point for other similar datasets.

\end{document}